\documentclass{article}
\usepackage{amsmath, amsthm, amsfonts, amssymb, enumerate}
\usepackage[pdftex]{graphicx}
\usepackage{samepaper, samemath, samepage}
\usepackage{cite}
\usepackage{tikz}
\newcommand{\eps}{\epsilon}
\newcommand{\vp}{\vec p}
\renewcommand{\r}{\vec r}
\newcommand{\X}{\mathcal X}
\newcommand{\Y}{\mathcal Y}
\renewcommand{\Z}{\mathcal Z}
\renewcommand{\P}{\mathcal P}
\renewcommand{\F}{\mathbb F}
\newcommand{\M}{\mathcal M}
\renewcommand{\C}{\mathcal C}

\title{Challenges in Bayesian Adaptive Data Analysis}
\author{Sam Elder\\MIT}

\begin{document}

\maketitle

\begin{abstract}
Traditional statistical analysis requires that the analysis process and data are independent. By contrast, the new field of adaptive data analysis hopes to understand and provide algorithms and accuracy guarantees for research as it is commonly performed in practice, as an iterative process of interacting repeatedly with the same data set, such as repeated tests against a holdout set. Previous work has defined a model with a rather strong lower bound on sample complexity in terms of the number of queries, $n\sim\sqrt q$, arguing that adaptive data analysis is much harder than static data analysis, where $n\sim\log q$ is possible. Instead, we argue that those strong lower bounds point to a limitation of the previous model in that it must consider wildly asymmetric scenarios which do not hold in typical applications.

To better understand other difficulties of adaptivity, we propose a new Bayesian version of the problem that mandates symmetry. Since the other lower bound techniques are ruled out, we can more effectively see difficulties that might otherwise be overshadowed. As a first contribution to this model, we produce a new problem using error-correcting codes on which a large family of methods, including all previously proposed algorithms, require roughly $n\sim\sqrt[4]q$. These early results illustrate new difficulties in adaptive data analysis regarding slightly correlated queries on problems with concentrated uncertainty.
\end{abstract}

\section{Introduction}

The growing field of adaptive data analysis seeks to understand the problems that arise from the way that many analysts study data: by an adaptive process of iterative measurements. Researchers often make those measurements on the same data set, inadvertently breaking an important assumption of previous statistical guarantees: the data and the reported measurement process are no longer independent. Rather than requiring this independence by fiat (e.g. mandating pre-registration of experiments or only using a holdout database once an ``exploratory'' phase of the analysis is complete), the aim of the nascent field of adaptive data analysis is to understand how accuracy decays under adaptive measurement, and build algorithms for extending that accuracy.

There is hope that such an understanding would solve one component of the current replicability crisis in experimental science. In practice, researchers often pick their experimental or data analysis techniques after observing the data, meaning any such results may no longer hold in a replication study with fresh, independent data. New statistical techniques will not be able to prevent dishonest researchers from cutting corners, but they could at least help honest researchers find true effects more reliably than current practice and more cheaply than replication studies.

Dwork, Feldman, Hardt, Pitassi, Reingold, and Roth (hereafter DFHPRR) formulated this problem in 2014 in a seminal paper with different components published in \emph{NIPS} \cite{dwork2015generalization}, \emph{Science} \cite{dwork2015reusable}, and \emph{STOC} \cite{dwork2015preserving}. After introducing the problem, they proposed several approaches based on algorithms from differential privacy, a strong stability guarantee. Subsequent analysis by Bassily, Nissim, Smith, Steinke, Stemmer, and Ullman (hereafter BNSSSU) \cite{bassily2015algorithmic} improved, streamlined and generalized this approach.

Simultaneous work on lower bounds offers to explain the difficulties inherent in adaptive data analysis. Hardt and Ullman \cite{hardt2014preventing} and Steinke and Ullman \cite{steinke2014interactive} used a construction from privacy known as interactive fingerprinting codes to argue a nearly matching lower bound on the sample complexity under adaptive queries. The results are quite strong, but as we will argue in section \ref{previouslower}, they take advantage of an unnatural information asymmetry in the original problem that does not arise in typical applications.

To better understand other challenges to adaptivity, we translate the original problem to a Bayesian context with a public prior. This naturally obviates the previous lower bound techniques, and also allows us to include other information about the population that might be known from the experimental design or prior publicly released data.

All of the previous techniques defined for the original problem have natural analogues in the Bayesian context, with the role of the empirical mean now played by the posterior mean. Having obviated the previous lower bound techniques, the natural first question is what other obstacles to efficient adaptive data analysis exist, and whether the analogues of state-of-the-art techniques are successful in eliminating those difficulties.

In Theorem \ref{main}, we answer this question by introducing a new problem on which all analogues of previously proposed techniques fail to achieve the optimal sample complexity. Rather than taking advantage of information asymmetry, this difficulty relies upon two important components: A difficult learning problem based on error-correcting codes with unusually high uncertainty in one direction, and a technique for using nearly orthogonal measurements to extract information about tiny effects from a general family of obfuscation techniques, including those using noise, rounding, and proxy mechanisms.

This result illustrates a new type of challenge faced by the usual approaches to adaptive data analysis. At the same time, this difficulty is not quite as constraining as those that exploit information asymmetry, which suggests that adaptive data analysis as commonly practiced might actually be easier than those previous lower bounds suggested.

While we study the full Bayesian context with a completely specified and accurate prior for clarity, we also think of this formulation as a heuristic analysis tool for understanding the original frequentist problem in situations where there is no risk of information asymmetry. Therefore, the message of this work is not to encourage researchers to attempt to write down accurate priors and translate every problem into a Bayesian context, but instead to use analogous techniques such as regularization to appropriately take into account prior beliefs.

\subsection{The Original Adaptive Problem}

As formulated by DFHPRR \cite{dwork2015preserving},\footnote{Notation has been slightly changed, because their results use variables $\eps$ and $\delta$ as they are used in differential privacy, while we use those variables in the more standard randomized algorithm notation (e.g. as used in PAC learning).} adaptive data analysis is a game between two players, the \emph{curator} and the \emph{analyst}. The analyst is given a distribution $\vp$ on a universe $\X$, while the curator only receives $n$ samples from $\vp$.

The analyst then asks him\footnote{For clarity, throughout this paper, we will refer to the curator using male pronouns and the analyst using female pronouns.} $q$ \emph{statistical queries} or \emph{parameters}: For some function $f:\X\to[0,1]$, what is $\E_{\vp}(f)$? In the case that $f$ only takes on values 0 or 1, these are known as \emph{counting queries} and correspond to asking the probability of some event $f^{-1}(1)\subset\X$. Either way, all true answers are within $[0,1]$, which provides an appropriate normalization within which to discuss error.

His goal is to answer every query to within an additive error of $\eps$ on the true population, or \emph{$\eps$-accurately}, with probability at least $1-\delta$ (over both the sample randomness and any randomness that he introduces). The problem is summarized in the following table. The central question is: How many samples $n$ does he need to do this, as a function of the parameters $q$, $\eps$, and $\delta$? Equivalently, given $n$ samples, how many queries $q$ can he answer successfully, as a function of $n$, $\eps$, and $\delta$?

\begin{center}\fbox{\begin{minipage}{4.5in}\textbf{Adaptive Data Analysis: Original Frequentist Version}

Two players: Curator and Analyst

Both receive problem parameters: $0<\eps,\delta<1$, $n\in\N$ and universe $\X$.

Analyst receives distribution $\vp$ over $\X$, unknown to the curator.

Curator receives $n$ independent samples from $\vp$.

(Analyst asks query $f:\X\to[0,1]$.

Curator replies with answer $a\in\R$.

Answer is $\eps$-accurate if $\abs{\E_{x\sim\vp}f(x)-a}<\eps$.)

Repeat the interaction in parentheses for $q$ total queries.

Curator wins if all answers are $\eps$-accurate with probability $1-\delta$.\end{minipage}}\end{center}

How does this game correspond to data analysis in the real world? The usual story goes like this: The analyst represents a powerful machine learning algorithm, aiming to find a very good fit to the true distribution. She will generally try to come to the strongest conclusion possible, which will probably involve difficult queries which are likely to reveal the most information, or those on which the curator is likely to be far off. To prevent this overfitting in the worst case, we model the analyst adversarially, which means that she could even know the distribution $\vec p$. The curator algorithm provides a layer of protection to the data set and answers these queries in order to make sure that all of the information the analyst gets is $\eps$-accurate.

This might not perfectly describe how data analysis is done in every application, but there is a domain where it is a decent fit: machine learning competitions. In a typical competition, the administrators often randomly split a data set into three similarly-sized components: a training set given to competitors, a test set reserved for official scoring at the end, and a holdout set to allow the competitors to try out their learning algorithms throughout the competition. In some circumstances, such as classification learning, the submission scores are actually counting queries corresponding to the probability of misclassification.

In this context, the official scoring using the unseen test set is an instance of \emph{static data analysis}, the traditional domain where all of the queries (submissions) are specified before results are announced. However, the submissions measured throughout the competition repeatedly using the holdout set are necessarily and intentionally adaptive; competitors are supposed to use their scores to improve their algorithms. Adaptivity can even take place between competitors as they collaborate to produce the best blended approach. Therefore, the traditional guarantees fly out the window and the public (unofficial) leaderboard throughout the competition is often inaccurate (see, e.g. \cite{wind2014learning}, originally cited in \cite{dwork2015preserving}).

Of course, as stated, it probably seems quite strange that the analyst also receives the true distribution $\vp$ if she is supposed to be seeking to learn it. If competitors already know the entire distribution, there is no need to hold the competition! We'll save this important criticism for later, though, first focusing on the static case, which describes the possible guarantees on the official leaderboard of a competition.

\subsection{Static Data Analysis}

If the analyst chooses all query functions $f_1,f_2,\dotsc,f_q$ before hearing the curator's answers to any of them, her strategy is said to be \emph{static}. Under static data analysis, the queries and data are independent, and a very simple curator strategy achieves remarkable accuracy: the empirical mean. This curator strategy simply answers query function $f$ with $\frac1n\sum_{i=1}^nf(x_i)$, where the $\{x_i\}_{i=1}^n$ are the data points.

Here's why this works: Each query is a bounded random variable, and each sample the curator receives gives an independent observation of each of them. Therefore, a standard Hoeffding and union bound give a probability of error greater than $\eps$ on any query of $\delta=q\exp(-\Omega(n\eps^2))$. Translating this into a sample complexity, we have a \emph{static sample complexity} of
\begin{equation}n=n_s(q,\eps,\delta):=O\(\dfrac1{\eps^2}\log\dfrac q\delta\right).\label{static}\end{equation}

This bound is tight for static data analysis, with a fairly simple matching example:
\begin{ex}[$\eps$-biased coin]Suppose a biased coin has either a probability $p=\frac12+\eps$ or $p=\frac12-\eps$ of heads. The curator receives the results of $n$ coin flips, and the analyst asks the query, ``What is the probability of heads?'' In our notation, the query function $f:\{H,T\}\to[0,1]$ is given by $f(H)=1$ and $f(T)=0$, so $\E_xf(x)=\Pr[x=H]=p$. Since the curator must answer within $\eps$ additive error, he must distinguish between the two cases using his data.

The number of heads out of all $n$ flips is a Bernoulli random variable with mean $np$ and variance $np(1-p)\approx n/4$, so by the Central Limit Theorem, this will fall on the wrong side of $n/2$ with probability $\sim\exp(-4n(p-1/2)^2)=\exp(-4n\eps^2)$. This is less than $\delta$ if $n\ge\Omega\left(\frac1{\eps^2}\log\frac1\delta\right)$, matching the upper bound achieved by the empirical mean, up to constants.\end{ex}

This example can also be easily extended to $q$ static queries:
\begin{ex}[$q$ copies of $\eps$-biased coins]In this problem, we have $q$ independent copies of the $\eps$-biased coin, and the analyst queries the probability of heads on each of them one by one. To write these queries as functions $f_i:\{H,T\}^q\to[0,1]$, we say $f_i(\vec x)=1$ if $x_i=H$ and $0$ if $x_i=T$. Since all of these queries are specified in advance, this is a static analyst strategy.

Each query is still $\eps$-accurate with probability $\sim\exp(-4n\eps^2)$, and if this is greater than $2\delta/q$, one of them will be wrong with probability greater than $1-(1-2\delta/q)^q>1-e^{-2\delta}>\delta$. This results in the same bound up to constants for $q$ queries: $n\ge\Omega\left(\frac1{\eps^2}\log\frac q\delta\right)$.\end{ex}

This isn't an unusual challenge, either; estimating any independent probabilities (not very close to 0 or 1) to within an $\eps$ additive error will require a sample complexity within a constant factor of $n_s$. In the final private leaderboard of a machine learning competition, therefore, the administrators can accurately score an exponential number of submissions in the number of data points, an excellent dependence.

But when we move to the adaptive setting, such as when updating the public leaderboard of a competition, these upper bound guarantees no longer hold. Of course, adaptivity always gives the analyst more options, so at least $n_s$ samples are necessary. But is that number still sufficient? Can the best curator algorithms still answer exponentially many adaptive queries?

\subsection{Previous Adaptive Lower Bounds}\label{previouslower}

Initially, the answer appeared to be a strong no. The existing literature has produced an array of very strong lower bounds, arguing for a wide gap between the adaptive and static problems. However, we will argue these all rely on that strange feature of this problem, that the analyst receives the true distribution, opening the possibility of a narrower gap without that feature.

First, let us describe the lower bounds. The simplest is found in an appendix of DFHPRR \cite{dwork2015preserving}, viewable in the arXiv preprint. The authors describe a model over $\R^d$ where the empirical mean algorithm fails, motivating their proposed variations on it. The distribution is $N(0,I_d)$, a zero mean $d$-dimensional Gaussian, and the analyst first queries the dot product with each of the standard basis vectors (rescaled and truncated to fall within $[0,1]$). After discovering where the sample is biased, she then queries along another diagonal direction, chosen to compile the errors of the first $d$ queries to produce an error that is typically $\sqrt d$ times larger. Since the empirical mean error is proportional to $\frac1{\sqrt n}$ and this requires $q=d+1$ queries, this means that $n=\Omega(q)$ is necessary for the empirical mean to be constant-accurate. Linear query dependence is awful: It can also be achieved trivially by looking at a fresh batch of data for every query, showing that the empirical mean massively underperforms once adaptivity is allowed.

Similarly, Blum and Hardt \cite{blum2015ladder} describe another adaptive attack on the empirical mean in a slightly different but a bit more general setting. Instead of querying a Gaussian along coordinate axes, they produce a series of random queries before again aggregating the results to produce a query on which the empirical mean will be $\sqrt{q/n}$-inaccurate. This again shows that the empirical mean strategy is deficient; it can only answer a linear number of queries.

Moving beyond attacks specific to the empirical mean, the most frequently cited general lower bound constructions due to Hardt and Ullman \cite{hardt2014preventing} and Steinke and Ullman \cite{steinke2014interactive} build on a long literature of privacy-preserving algorithms, and in particular on an attack known as interactive fingerprinting codes. At a high level, the analyst asks queries that can only be answered successfully if the curator has seen particular data points, and in so doing is able to reconstruct the data that the curator has seen and query the remainder that he hasn't seen. This is harder than it sounds at first, because the analyst must use queries that force the curator to reveal knowledge of a particular point or answer $\eps$-inaccurately, not just $1/n$-inaccurately as the empirical mean frequently does. This attack alternatively requires a common cryptographic assumption (one-way functions) and only applies to computationally bounded adversaries, or requires high ($d\sim n^2$) dimension, but in either case, it can be achieved with only $O(n^2)$ queries, which is remarkable given its generality.

These results imply the improved estimates of BNSSSU \cite{bassily2015algorithmic} are nearly tight for this problem. This would seem to be the end of the story: In adaptive data analysis, the curator can answer only quadratically many queries, far fewer than the exponential number of queries that can be answered in static data analysis, at least in high dimension.

However, these strong lower bound examples rely on a key information asymmetry between the analyst and curator that we will now argue is unrealistic. In both the Gaussian and fingerprinting cases, the analyst in fact knows \emph{the exact true distribution}, and that the curator is left guessing from the data. This feature is critical to those constructions: In the Gaussian example, the analyst must know the true answer (or in other words, where the origin is) to be able to determine in which direction the analyst's answers are wrong. In the interactive fingerprinting attack, the analyst must know the possible samples that the curator could see so she can construct queries designed to test for them. Even the boosting attack takes advantage of this information asymmetry by limiting what the curator can do; see section \ref{indclass} for the full details.

In the picture of how this would be used in machine learning, though, this asymmetry is unrealistic. The analyst is trying to learn something about the distribution from the data, but there's no learning to be done when she already knows it. In a competition, if some competitor already knows the full distribution of the data, there is no need to keep any data in a holdout set or for that matter, even to have them enter the competition in the first place.

Finally, Nissim and Stemmer \cite{steinkeprivate} have recently attempted to make the interactive fingerprinting attack information-symmetric by encrypting it with public key encryption. From one perspective, this gives a computational lower bound, arguing that computationally bounded curators can not answer more than $O(n^2)$ queries even in the information-symmetric setting. At the same time, the problem description in terms of public keys and private keys is exponentially long, and the key step is that the computationally bounded curator does not have enough time to read this entire description and do proper inference to the private keys. This result sheds some light on the problem mathematically, but ultimately just hides the same information-asymmetric attack behind a computational barrier. We are interested in new categories of difficulties to adaptivity that don't simply hinge on being unable to do inference, so we will generally consider both players to be computationally unbounded.

\subsection{Summary of Results}

The first main contribution of this paper, in Section \ref{bayes}, is a new Bayesian formulation of the problem that incorporates information symmetry via a public prior. This reopens the question of how many queries can be answered in the adaptive setting, because the previous nearly matching lower bounds no longer apply.

All of the previous curator algorithms proposed by DFHPPR to this problem were based in some way on obfuscating the empirical mean, whether with noise (the Laplacian method) or with a proxy distribution in low dimension (private multiplicative weights), and those translate neatly over. In the new Bayesian context, the corresponding baseline curator strategy is the posterior mean: update the prior $\P$ to the posterior $\P'$ according to Bayes' law based on the observed data, and output answers according to the mean answer of distributions chosen by that posterior.

The second main contribution of this paper is a new set of difficulties in the Bayesian context showing that this entire family of methods, obfuscating the posterior mean, falls short of achieving the static sample complexity \ref{static}. As these difficulties arise in the Bayesian context, they necessarily do not rely on information asymmetry, and therefore provide a new picture of the difficulties of adaptive data analysis that will be helpful for designing algorithms for real-world applications.

First, in Section \ref{pm}, we examine the posterior mean algorithm itself. Our negative result here provides a more helpful variant to the Gaussian example in DFHPRR, and also provides a starting point for the rest of the problems in the paper. In particular, we introduce the \emph{linear classification} model, very similar to the famously difficult problem of learning parities robustly, and an adaptive analyst strategy that will cause the posterior mean to answer inaccurately using only linearly many adaptive queries.

In Section \ref{npm}, we consider the analogue of the Laplacian technique: adding noise to all curator answers. This noise prevents the analyst from aggregating answers under the linear classification model, but it is still vulnerable to attack on similar models with more parameters. In particular, on higher degree polynomial classification, we show that noisy posterior requires a polylogarithmic sample complexity, strictly more than the logarithmic dependence of \eqref{static}.

In Section \ref{rpm}, we introduce a new obfuscation technique possible with a prior: rounding the posterior mean answers in a prior-sensitive fashion, in order to minimize the information the analyst could gain. This does not fall easily to polynomial classification, but trouble is again just around the corner: we construct another model based on an error-correcting code that an analyst can again exploit with the same number of queries as against noisy posterior. We spend all of Section \ref{extracting} building up this problem, and in the end, we prove the following general result:
\begin{thm}[Informal, see Theorem \ref{main}]On a certain high-dimensional problem and against a particular analyst strategy, if the curator answers every query with any function of the posterior mean on that query, he will only be accurate on up to $O(n^4\log n)$ queries.\end{thm}
As we note in Corollary \ref{pmwmonster}, this attack also applies to the analogue of the private multiplicative weights algorithm, which will default to its fallback method on nearly every query.

While this result is not as general as the bounds of the original problem, it still introduces a substantial obstacle for a large family encompassing all previously proposed methods and their natural analogues. This problem also gives us a new picture of the difficulties in adaptive data analysis beyond those due to information asymmetry. In short, none of the known obfuscation techniques are effective at preventing information leakage on the slightly correlated queries that the analyst in the proof of Theorem \ref{main} uses.

Finally, in Section \ref{disc}, we explore the consequences of this result and pick out some potential threads that might lead to more successful curator algorithms against this new difficulty, outside of the general framework of obfuscating the empirical or posterior mean.

\section{Bayesian Adaptive Data Analysis}\label{bayes}

In this section, we introduce the new problem of Bayesian adaptive data analysis. Just two lines are changed:

\begin{center}\fbox{\begin{minipage}{4.5in}\textbf{Adaptive Data Analysis: New Bayesian Version}

Two players: Curator and Analyst

Both receive problem parameters: $0<\eps,\delta<1$, $n\in\N$ and universe $\X$.

Both also receive a prior $\P$ over distributions on $\X$.

A distribution $\vp$ is chosen from $\P$, unknown to both curator and analyst.

Curator receives $n$ independent samples from $\vp$.

(Analyst asks query $f:\X\to[0,1]$.

Curator replies with answer $a\in\R$.

Answer is $\eps$-accurate if $\abs{\E_{x\sim\vp}f(x)-a}<\eps$.)

Repeat the interaction in parentheses for $q$ total queries.

Curator wins if all queries are $\eps$-accurate with probability $1-\delta$.\end{minipage}}\end{center}

In this way, we prevent the analyst from employing strategies that rely on side knowledge of the distribution, by giving all such knowledge to the curator as well. This is the main point; we are not claiming that an accurate public prior can be written down for every problem in practice. Instead, by examining the situations where it can, we hope to explore the possible difficulties that do not arise simply from exploiting side information. The hope is that observations in this domain will naturally translate to heuristics in real-life scenarios where information symmetry in an informal sense is appropriate to assume.\footnote{All of that said, this problem as stated is also of interest. It is analogous to the area of Bayesian analysis focused on finding freqeuntist properties (like posterior convergence) of Bayesian models (see, e.g. the survey paper \cite{rousseau2016frequentist}).}

Future work could consider the case where these $\P$ and $\P'$ are close in some suitable sense, in order to measure the difficulties with a small amount of side information, but as a first contribution, we wish to study difficulties that still arise without side information.

The prior $\P$ may be discretely or continuously supported (within the simplex of probability distributions $\Delta(\X)$), and the examples we construct will feature both. We call a particular universe $\X$ and prior $\P$ a \emph{model} $\M=(\X,\P)$.

As designed, this assumption of information symmetry obviates all of the previous potential lower bound models. If the analyst knows the distribution exactly, so does the curator, and then the curator can simply give the exact answer without referencing any samples. And if the analyst is in the dark, she can't determine when the curator's answers are inaccurate to aggregate the errors.

\subsection{The Gaussian Example with a Prior}\label{gaussian}

To illustrate the improvement of this method, we reexamine the Gaussian example from DFHPRR. Recall that the setup of this problem involves a spherical $d$-dimensional Gaussian with known variance but unknown mean, and the analyst is seeking to find a direction in which the curator will give inaccurate answers.

To translate it into a Bayesian framework, we need a prior over the mean, since we can no longer make it known to the analyst but not the curator. For simplicity in computation, suppose the true center $c$ of the distribution is distributed as $N(0,\sigma I_d)$ for some $\sigma>0$, and each data point is known to be generated from adding $N(0,I_d)$ to $c$. This normalization is appropriate, since dot products with unit vectors will likely deviate from the true mean by a constant, so only truncation of a constant fraction of the space is necessary to make such queries fall within $[0,1]$. Therefore, up to constants, we can assume that querying the dot product with any unit vector is allowed.

In this context, $\sigma=0$ corresponds to the mean being completely known (at the origin), and large $\sigma^2\gg1/n$ corresponds to a widely diffused prior for which the data will be needed to clarify the position. As a sum of Gaussians, it is easy to compute that the empirical mean $\hat c$ will be distributed as $N(0,(\sigma^2+\frac1n)I_d)$. For notational convenience, choose a new basis for $\R^d$ so that $\hat c=te_1$.

Updating to the posterior, another easy computation shows that the probability density function for the center $c$ is proportional to
\begin{align*}
\exp\left(-\frac{\norm c^2}{2\sigma^2}-\frac{n\norm{c-\hat c}^2}2\right)&\propto\exp\left(-\frac{n+1/\sigma^2}2\left(c_1-\frac n{n+1/\sigma^2}t\right)^2-\frac{n+1/\sigma^2}2\sum_{i>1}c_i^2\right)\\
&=\exp\left(-\frac{n+1/\sigma^2}2\Norm{c-\frac n{n+1/\sigma^2}\hat c}^2\right).
\end{align*}
The posterior mean will then answer all queries for the mean with the center of this distribution, $\dfrac n{n+1/\sigma^2}\hat c$. Because this is now a Gaussian with variance $\frac1{n+1/\sigma^2}$, the probability of $\eps$ error in any query direction is $\sim\exp(-\eps^2(n+1/\sigma^2)/2)<\exp(-n\eps^2/2)$, matching the static bound \eqref{static} up to constants.

In other words, all of the difficulty in the Gaussian example was due to the information asymmetry. This is most clear when we vary $\sigma$. As $\sigma\to\infty$, the posterior mean approaches the empirical mean, but the variance of the posterior is still bounded by $1/n$, so it remains accurate. As $\sigma\to0$, the posterior mean takes less account of the data and approaches the origin. The original example corresponds to $\sigma\to\infty$ for the curator's prior while $\sigma\to0$ for the analyst's prior, the maximal information asymmetry in this problem.

The same conclusion holds for the interactive fingerprinting attack, because the analyst no longer knows which fingerprints of data points to look for. The essential query in that problem relied on the data points that the analyst knew were possible but the curator didn't, and with information symmetry, there are no such data points.

We are left with only the basic lower bound examples we considered in the static case, which do translate nicely into this setting, since that analyst's queries don't require knowledge of the distribution. For the $\eps$-biased coin, we can consider the prior to be uniform on the two cases, and similarly for $q$ copies of it. In these cases, it is still necessary to see enough data to reliably distinguish probabilities of $\frac12\pm\eps$ from each other, so \eqref{static} still holds.

Is that really enough, though? This is the key question: Under information symmetry, can the static bound \eqref{static} be achieved for adaptive queries? If not, what new attacks can the analyst employ, and what bounds do those place on query complexity?

\section{Posterior Mean}\label{pm}

We begin by analyzing the standard algorithm, the analogue to the empirical mean in the Bayesian setting. Like with empirical mean, this will not work on some difficult problems, but those problems are more involved than the Gaussian example proposed to fool the empirical mean. First, let us formally define the algorithm.

\begin{defn}[Posterior Mean]Suppose (for clarity) that $\X$ is finite and that the prior $\P$ puts weights of $w_j$ on finitely many discrete hypotheses $\vp_j$, for $j=1,\dotsc,l$. Let $\vp_j(x)$ be the probability of getting data point $x$ under hypothesis $\vp_j$. After observing samples $x_1,\dotsc,x_n$, the posterior mean curator algorithm first calculates the \emph{posterior} $\P'$ by updating the weights of each hypothesis according to Bayes' rule:
\[w_j'=\frac{w_j\prod_{i=1}^n\vp_j(x_i)}{\sum_{j=1}^lw_j\prod_{i=1}^n\vp_j(x_i)}.\]
The algorithm then averages the answer to the query under each of these hypotheses according to these new weights, answering with
\[\E[f(\P')]:=\sum_{j=1}^lw_j'\E_{x\sim\vp_j}[f(x)].\]\end{defn}

In the case that the prior is continuously defined, we replace the sum with an integral and terms like $\vp_j(x_i)$ with the appropriate probability density. Most of our examples will be finite, though.

We can see that this represents a very reasonable attempt to approximate the correct answer: The posterior represents the correctly updated beliefs of the curator, and the mean simply aggregates the results.

\subsection{Classification Models}\label{classification}

Unfortunately, we will eventually describe a model on which the posterior mean fails spectacularly. This example will be the first of a series of models which will be a thorn in the side of all posterior mean-based curator algorithms, so we begin by describing this class of models in general.

In a \emph{classification model}, the universe is a product $\X=\Y\times\Z$, where we think of $\Y$ as an underlying known population, and $\Z$ as a set of labels generated by some unknown function $\ell:\Y\to\Z$ that we are trying to learn.

That is, each hypothesis $\vp_j$ in the support of the prior corresponds to some possible function $\ell_j:\Y\to\Z$. All hypotheses have the same marginal $\r$ on $\Y$, and after drawing a sample $y\sim\r$, output the point $(y,\ell_j(y))$. In other words, the hypothesis $\vp_j$ has the \emph{graph} of function $\ell_j:\Y\to\Z$ as its support, with weights determined by $\r$.

In the examples we consider in this paper, $\r$ will be uniform on $\Y$ and $\Z=\F_2$.\footnote{To write succinct equations that relate $\F_2$ to $[0,1]$ we will frequently abuse notation and conflate the two elements of $\F_2$ with the real numbers $0$ and $1$.} In equations, hypothesis
\[\vp_j(y,z)=\begin{cases}2^{-m}\quad& z=\ell_j(y)\\0\quad&{\text{otherwise}}.\end{cases}\]

One important feature of classification models to us is that they have very easy to understand posteriors. Each potential sample point $(y,z)$ occurs with probability either $r_y$ or $0$ under each hypothesis. Therefore, the posterior remains uniform over all functions which are consistent with all of the observed samples, and puts a zero weight on any that are inconsistent with even one. Call the hypotheses that agree with all of the samples \emph{eligible}.

\subsection{Independent Classification}\label{indclass}

Before introducing the problematic case for the posterior mean, let us first look at an example where the posterior mean does very well: the boosting attack of Blum and Hardt \cite{blum2015ladder} against the empirical mean.

Suppose that all possible label functions $\ell_j:\Y\to\{0,1\}$ are equally likely. This is a somewhat trivial ``learning'' scenario, since the curator only learns the value of the function at each point he sees, and nothing more. Still, it illustrates an important difference between the posterior mean and empirical mean when it comes to boosting and reconstruction-style attacks.

The attack as described in \cite{blum2015ladder} works as follows: For the first $q-1$ queries, the analyst picks random functions $\ell_i':\Y\to\{0,1\}$ and asks for their correlation with the true labeling function. To write this in the notation of statistical queries, her $i$th query function is the indicator function on the graph of $\ell_i'$: $f(y,z)=\delta_{z,\ell_i'(y)}$. The expected value of this function is therefore the probability that $\ell_i'$ and the true labeling function $\ell$ agree on a uniformly random chosen $y\in\Y$.

The analyst then collects all queries with agreement greater than $1/2$, according to the curator. She defines a new query by taking the majority label among all of these queries, which will be biased to agree with $\ell$ more frequently on the points that the curator has seen.

Suppose that $k$ of the $n$ data points have labels that match $\ell_i'$. The empirical mean will simply answer the query with that fraction, $k/n$. The posterior mean is more subtle, though: It averages the answers over all functions consistent with the data. On the $n$ data points that he's seen, the average is the same $k/n$, but on all $\abs\Y-n$ data points he hasn't, the average agreement is just $1/2$. Therefore, the posterior mean regularizes the empirical mean's answer back towards the prior, answering with
\[\frac12\frac{\abs\Y-n}{\abs\Y}+\frac kn\frac n{\abs\Y}=\frac12+\frac{2k-n}{2\abs\Y}.\]

We can actually easily verify that the posterior mean makes the appropriate inference against random queries. The true agreement between the random $\ell_i'$ and the true label function on unseen points is simply a rescaled binomial random variable with mean equal to the posterior mean and variance $\frac{\abs\Y-n}{\abs\Y^2}<\frac1n$. By Hoeffding's inequality, this differs by $\eps$ from its mean with probability $\sim\exp(-\Omega(n\eps^2))$, satisfying the desired bound.

The empirical mean's answer is also accurate against the random queries, but the boosting attack distinguishes them. Both answers to the random queries fall on the same side of $1/2$, so the attack constructs the same biased query for each of them. But when the empirical mean answers with a correlation that is $\sqrt{\frac qn}$ too high, the posterior mean only goes up by $\sqrt{\frac qn}\frac n{\abs\Y}=\frac{\sqrt{qn}}{\abs\Y}$, much less.

In fact, this is the correct increase. The above concentration argument for the posterior mean's accuracy actually applies to any query that is uncorrelated with the true labels off of the data known to the curator. It is easy to see that this holds for the biased query, and in fact, any query that the analyst or curator could construct, since neither knows anything about the true labels off of the known data.

This is a subtler use of information asymmetry than the usual case, since the boosting attack makes no reference to the true labels. Instead, the success of the attack depends on the true support $\abs X$ being significantly larger than $n$, a publicly known fact which the empirical mean curator ignores. Once the curator is allowed to take the support size into account and regularize towards $1/2$, he no longer makes this type of mistake.

As a further example, we could consider what would happen if the support size is instead drawn from some nontrivial prior. This is somewhat similar to the classical problem of estimating support size of a distribution (see e.g. \cite{valiant2011estimating}). However, the curator does not need to accurately estimate $\abs\Y$, but rather, $1/\abs\Y$ to an accuracy of roughly $1/n$.

In particular, if all of the data points are unique, then the curator probably (depending on his prior) learns that $\abs\Y$ is likely much larger than $n$. This makes his estimate for $1/\abs\Y$ negligible, so he will answer (very close to) the prior mean of $1/2$ for every query, even the boosted ones. This illustrates how important information present \emph{in the data} can be ignored by the empirical mean, a deficiency that the boosting attack exploits.

While a Bayesian curator takes this into account perfectly, this also indicates morally that other regularization techniques would also prevent this sort of ``learning.'' Indeed, in this circumstance, it is clear that any sort of cross-validation would properly evaluate this boosting algorithm if it attempted to learn in this fashion on the training data. Fortunately, many researchers already apply such techniques to keep themselves from overfitting. Instead of dwelling on those techniques, we move on to learn about what models are hard even in the Bayesian context.

\subsection{Linear Classification}

What happens if the labels are not independent? For our next example, also known as the problem of learning parities, we consider the uniform prior over linear classification functions $\ell_j$ on $\F_2^m$, which we call $LC_m$. In other words, uniformly random coefficients $a,b_1,b_2,\dotsc,b_m\in\F_2$ are chosen, and the distribution is uniform on the graph of the function $\ell(x_1,\dotsc,x_m)=a+\sum_kb_kx_k$.\footnote{This model is a slight variant on the classical problem of \emph{learning parities}. Technically, we are looking at degree at most 1 polynomials, only half of which are linear in the linear algebraic sense (those with $a=0$). We do this to introduce a symmetry: Every point in $\Y\times\Z$ is equally likely, rather than making $(\vec 0,1)$ impossible.}

This model also has an easy-to-understand posterior:
\begin{itemize}
\item From a series of samples $\{(y_i,z_i)\}_{i=1}^n$, we can construct the function restricted to the affine span of $\{y_i\}$. At any point outside of the affine span, the function is equally likely to be 0 or 1.
\item Considering the samples in order, call $(y_i,z_i)$ \emph{novel} if $y_i$ is not in the affine span of $y_1,\dotsc,y_{i-1}$. Each novel sample cuts the number of eligible hypotheses in half, by symmetry.
\item If $i$ points are affinely independent, their affine span has size $2^{i-1}$. Therefore, each novel sample also doubles the size of the affine span of the samples.
\item The probability that the first $m$ samples are all novel is hence
\[1-\frac1{2^m}-\frac2{2^m}-\frac{2^2}{2^m}-\dotsb-\frac{2^{m-2}}{2^m}=\frac12+2^{-m}.\]
\end{itemize}

With these observations, we can prove that the analyst has a winning strategy if the curator gets a precise number of samples:
\begin{thm}Under model $LC_n$, there is an adaptive analyst strategy which causes the posterior mean curator strategy to answer $\frac14$-inaccurately with probability $>\frac12$, using only $n+2$ queries.\label{linearpm}e\end{thm}
\begin{proof}To be clear, we take $n=m$, so that as we've computed, with probability $\frac12+2^{-m}>\frac12$, all of the samples are novel. After $m$ novel samples, the number of eligible hypotheses is down to two, which agree on half of $\F_2^m$ and disagree on the other half. The posterior puts weights of $1/2$ on each of these, so the posterior mean puts a weight of $2^{-m}$ on the $2^{m-1}$ known points and $2^{-m-1}$ on each of the $2^m$ unknown possible points.

Therefore, by querying the indicator function on individual points $\{(y,z)\}$, the analyst has access to an oracle for whether the curator knows the value of the function there. By sampling an affine basis for the entire space (such as $0,e_1,\dotsc,e_m$), she can compute that affine span with only $m+1=n+1$ queries, and determine which two hypotheses are still eligible.

Then she can exploit the remaining ignorance of the curator by querying the entire graph of one of those hypotheses. The correct answer will be $\frac12$ or $1$, equally likely, so the posterior mean answer of $\frac34$ is $\frac14$-inaccurate, as desired.\end{proof}

Note that this argument shows that this particular curator cannot even answer more than linearly many queries, which we've already mentioned is trivial. In fact, the analyst technically doesn't even need all $n+1$ exploratory queries:
\begin{thm}Under model $LC_n$, there is an adaptive analyst strategy which causes the posterior mean curator strategy to answer $\frac14$-inaccurately with probability $>\frac12$ using only two queries.\label{ternary}\end{thm}
\begin{proof}Consider some enumeration $y_1,y_2,\dotsc,y_{2^m}$ of $\F_2^m$ and the query function defined by $f(y_l,1)=2\cdot 3^{-l}$ and $f(y_l,0)=0$. The posterior mean on this query is therefore $2^{-m}\sum_{l=1}^{2^m}g(y_l)3^{-l}$, where $g(y)=0$ if the curator knows that $f(y)=0$, $g(y)=1$ if the curator doesn't know $f(y)$, and $g(y)=2$ if the curator knows that $f(y)=1$. Therefore, by reading off the digits of the ternary expansion of $2^m$ times the posterior mean to this query, the analyst can determine everything the curator knows. She can then query one of the two eligible hypotheses, and the curator will be $1/4$-inaccurate as before.\end{proof}

This isn't a matter of posterior inference being the wrong thing to do here; it's easy to see that a variant of the attack in Theorem \ref{ternary} would be able to find the data points if the curator decided to ignore the prior and use the empirical mean on this problem instead. The important feature is that there is a query that the curator does not have the information to answer reliably accurately, whether he uses all of the information he does have or not.

For intuition's sake, it might help to look at the way the uncertainty of the curator evolves as he gets more data points. Initially, he is slightly uncertain in nearly every direction, but as he eliminates hypotheses, that uncertainty is reduced in most directions while increasing in a smaller number. When he reaches the last two hypotheses, all of the remaining uncertainty is concentrated along the direction of their disagreement. This particular direction is one out of exponentially many equally likely possibilities originally, so the analyst can't simply guess it. But if she can learn what the curator knows, she can determine it and query it. This will be a running theme in future examples as well.

\section{Noisy Posterior Mean}\label{npm}

The standard response is that both the posterior mean and empirical mean are too precise: They unnecessarily give away information in unnecessary bits of precision. We will investigate several methods for obfuscation to attempt to prevent this leakage, but the first is by adding independent noise to every answer.

For simplicity, we will consider adding unbiased Gaussian noise. The first requirement on the noise is that it can't affect answers too much, i.e. be more than $O(\eps)$ with probability $1-\delta$. Therefore, we set its variance as $\frac1{4n}$, matching the sampling variance in the case of biased coins.

Through this obfuscation, this noisy posterior mean can indeed successfully answer exponentially many queries under the linear classification model:

\begin{thm}Under model $LC_n$, the noisy posterior mean curator strategy can answer $q$ queries $\eps$-accurately using $n=O\left(\frac1{\eps^2}\log\frac q\delta\right)$ samples.\label{npmlc}\end{thm}

Unfortunately, this result won't be the end of the story, so we don't want to dwell on it too long. Some features will be important later, so we give a sketch of the proof here and save the full proof for Appendix A.

\begin{proof}[Proof Sketch]The main idea is that the noisy posterior actually answers sufficiently similarly to the prior, and the prior answers sufficiently similarly to the true answer on almost every query. The second claim amounts to an insightful lemma that we include here:

\begin{lem}Under model $LC_m$, the Prior Mean curator strategy will answer any query $\eps$-accurately with probability at most $1-\dfrac{2^{-m}}{4\eps^2}$.\label{priorconc}\end{lem}
\begin{proof}The Prior Mean strategy is simple: by symmetry, it puts equal weights of $2^{-m-1}$ on every point in $\F_2^m\times\F_2$, and answers the average value of the function according to those weights. Abusing notation, we call the prior mean value $\E_\P(f)$.

To bound the difference of the real answer and the prior mean, define a new function $f':\F_2^m\to[-1,1]$ by $f'(y)=f(y,0)-f(y,1)$. Then we can rewrite
\[\E_{\vp_j}(f)-\E_\P(f)=\frac1{2^m}\sum_{y\in\F_2^m}\left(f(y,\ell_j(y))-\frac12(f(y,0)+f(y,1))\right)=\frac1{2^{m+1}}\sum_{y\in\F_2^m}f'(y)(-1)^{\ell_j(y)}.\]
The terms in this sum are independent, because for any two points $y,y'\in\F_2^m$, the ordered pair $(\ell_j(y),\ell_j(y'))$ is equidistributed among $\F_2^2$ as $\ell_j$ over all linear functions. So the variance of this deviation is the sum of the variances of the individual terms, or
\[\Var(\E_{\vp_j}(f)-\E_\P(f))=\frac1{2^{2m+2}}\sum_{y\in\F_2^m}f'(y)^2\le\frac{2^m}{2^{2m+2}}=\frac{2^{-m}}4.\]
The lemma follows immediately from Chebyshev's inequality.\end{proof}

The rest of the proof involves the somewhat counterintuitive step of analyzing how the posterior of the \emph{analyst} evolves, knowing the curator is using the noisy posterior mean algorithm. Initially, according to the prior, the analyst puts an equal weight on all of the hypotheses. We argue in an inductive fashion (as the queries come in one by one) that these weights remain approximately the same, for all but a small fraction of hypotheses that have lower weights, with high probability. In other words, the information that the analyst gets from answers to the queries is very likely to be sufficiently diffuse that he is unable to learn anything without exponentially many queries. Again, the full details are in Appendix A.\end{proof}

\subsection{Polynomial Classification}

However, noisy posterior mean's successes are rather short-lived, because it fails at the next possible instance, more general polynomial classification, which we introduce now.

For the prior of polynomial classification $PC_{m,k}$, uniformly randomly choose $M=1+m+\binom m2+\dotsb+\binom mk$ coefficients $c_{i,S}\in\F_2$ for $i=0,\dotsc,k$ and $S\subset[m]$ of size $k$, and let the distribution be uniform on the graph of the polynomial
\[\ell(x_1,\dotsc,x_m)=\sum_{i=0}^k\sum_{S\in\binom{[m]}k}c_{i,S}\prod_{j\in S}x_j.\]
Note that these are all possible polynomials of degree at most $k$, since $x_i^2=x_i$ over $\F_2$. We will typically think of $k\ll m$ or constant so $M=\Theta(m^k)$, but this definition is valid for any $1\le k\le m$.\footnote{For $k=m$, this is actually independent classification again.}

This model has some of the same properties as linear classification.
\begin{itemize}
\item By counting coefficients, there are initially $2^M$ eligible hypotheses.
\item Each sample introduces a linear constraint on the coefficients. If this constraint is not already known, it cuts the number of available hypotheses in half. As before, we call such samples \emph{novel}.
\item By the theory of Reed-Muller codes, any two degree $\le k$ polynomials differ on at least a $1/2^k$ fraction of $\F_2^m$. Therefore, each data point is novel with probability at least $1/2^k$, until there is only one eligible hypothesis left.
\end{itemize}

By Markov's inequality, this implies that with probability at least $1/2$, less than $2^{k+1}M$ data points are necessary to eliminate all but one eligible hypothesis. Right before the last novel sample, there must be exactly two eligible hypotheses remaining. Therefore, there exists some $M\le n<2^{k+1}M$ such that with probability $>\frac1{2^{k+2}M}\ge\frac1{2^{k+2}n}$, there are exactly two eligible hypotheses remaining after the curator receives $n$ data points.

In this case, despite the $n=\Omega(m^k)$ data points, the curator still can't distinguish two hypotheses that disagree on at least $1/2^k$ of the space. Therefore, if $\eps<\frac1{2^{k+1}}$, the curator will know a query which he cannot answer $\eps$-accurately with a probability greater than $1/2$.

We now describe an analyst attack that only uses $2^{O(m)}$ queries. Since this is far less than $\exp(n)=\exp(\Omega(m^k))$, this shows that the noisy posterior mean curator algorithm falls short of the static bound \eqref{static} here.\footnote{We will eventually provide a much stronger guarantee of this type on a different problem, but this example is perhaps easier to grasp.}

The analyst first repeatedly queries the indicator function on one individual point $(y,z)$. Recall that the posterior mean puts a weight $0$, $2^{-m-1}$, or $2^{-m}$ on that point according to whether the curator knows that $\ell(y)=z$ or $\ell(y)\neq z$. Noise with variance $\frac1n$ ordinarily drowns out this signal, but if the analyst asks $2^{2m}$ times, the average will have noise variance $\frac1{2^{2m}n}$, making these gaps about $\sqrt n$ standard deviations apart. Therefore, the average of these repeated queries will tell the analyst what the curator knows about $\ell(y)$ with probability $1-\exp(-n)$. Repeating this for all $y\in\Y$ only takes $2^{3m}$ queries in all, from which the analyst can recreate everything the curator knows, and can find that query on which he will be confused.

Taking $k$ as small as possible, we see that the analyst can answer at most $q=\exp(O(n^{1/\log_{1/2}(2\eps)}))$ queries, or equivalently, $n=\log^{\log_{1/2}(2\eps)}q$ data points are necessary to answer $q$ queries. To wrap all of this up, we have just shown:
\begin{thm}Under $PC_{m,k}$, where $n\sim\binom mk$, there is an adaptive analyst strategy which causes the noisy posterior mean curator strategy to answer $\eps$-inaccurately with probability at least $\frac\eps{2n}$, using only $2^{O\left(n^{1/\log_{1/2}(2\eps)}\right)}$ queries.\label{pcnpm}\end{thm}

We will later show a much stronger bound than this, so don't focus on the specifics of this result. Clearly this is not as spectacular of a failure as the linear (or two!) queries that can defeat posterior mean, or the quadratic bounds that can be achieved without information symmetry. On the other hand, it still establishes a clear gap between the sample complexity of this algorithm on the static and adaptive cases. We therefore turn our attention back to the algorithm side, with hopes of learning from this failure.

\section{Rounded Posterior Mean}\label{rpm}

Taking a step back, we've seen that adding noise is moderately effective because each answer gives very little evidence of whether the posterior is one thing or another. For clarity, we can quantify this evidence with the likelihood ratio:
\[\frac{e^{-x^2/2\sigma^2}}{e^{-(x-2^{-m})^2/2\sigma^2}}=e^{-2^{-m}(x-2^{-m-1})/\sigma^2}.\]
Since $2^{-m}/\sigma^2\sim\frac n{2^m}\sim\frac M{2^m}$ is fairly small, each query doesn't give the analyst very much evidence one way or the other.

The problem, of course, is that even this small amount of evidence can be accumulated with enough repeated queries. But what if the answers gave no evidence at all to distinguish these hypotheses?

This is the idea behind the rounded posterior mean family of curator strategies: Split the interval $[0,1]$ into subintervals of width less than $\eps$, and answer with the midpoint of the interval that contains the posterior mean.

In this case, if two potential posterior means lie within the same interval, the curator will give the same answer in either case, providing zero evidence in either direction. However, if the two potential answers straddle two intervals, the analyst can eliminate one of them.

If the interval boundaries are specified universally (say, at $\eps,2\eps,\dotsc$), then the analyst can easily construct queries with posterior means that straddle those boundaries. For instance, in $PC_{m,k}$, suppose there is a consistent boundary at $B<\frac12$, and consider the query function defined as
\begin{equation}f(y,z)=\begin{cases}B&\quad y\neq y'\\z&\quad y=y'.\end{cases}\label{slightdiagonalquery}\end{equation}
If the curator knows that $\ell(y)=0$, the posterior mean will lie in the interval below $B$, and otherwise, it will lie in the interval above $B$, so this query will tell the analyst whether the curator knows that $\ell(y)=0$.

On the other extreme, if the intervals depend on the data, the answers themselves will leak evidence. But in the Bayesian context, there is a third option: the intervals could depend on how the query function is answered by the prior, not the posterior. This allows the curator to avoid situations like \eqref{slightdiagonalquery} while allowing the analyst to predict the intervals, so they offer no new information.\footnote{This algorithm doesn't generally translate into an algorithm in the frequentist setting, but it's possible that the context of the problem could correspondingly give the curator natural boundaries that don't split many close hypotheses. In any case, we continue to analyze the strongest version of this we can construct, since there might be circumstances where such a prior is known.}

In particular, we can choose the intervals such that the regions near the boundaries have very little probability mass according to the prior. First, let us state a lemma that shows that this is possible.
\begin{lem}Let $D$ be a distribution on $[0,1]$. Then there exists a partition of the interval $0=x_0<x_1<x_2<\dotsc<x_m=1$ with $\eps/3<x_{i+1}-x_i<\eps$ such that $\forall\eta>0$,
\[\Pr_D\left[\bigcup_{i=1}^{m-1}(x_i-\eta,x_i)\right],\Pr_D\left[\bigcup_{i=1}^{m-1}(x_i,x_i+\eta)\right]<\frac{6\eta}\eps.\]
Moreover, if $D$ is discretely supported with support size $s$, there is an algorithm that can compute the partition in time $O(s^3/\eps)$.\label{safepart}\end{lem}
Why would this help? Such a partition will avoid sections of $[0,1]$ where $D$ places a large amount of probability mass. This decreases the likelihood of the boundaries leaking information. Notice that this result is tight for the uniform distribution up to the factor of $6$, because it puts a weight of $\eta$ on all of the intervals there.

This lemma is rather technical and particular to the Bayesian context, so we save its proof for Appendix B. Define the \emph{smart rounded posterior mean} to be the rounded posterior mean applied with these intervals.

As we might hope, this easily handles linear classification. This proof is insightful and short enough that we include it here.
\begin{thm}Under model $LC_n$, the smart rounded posterior mean curator algorithm answers $q$ queries $\eps$-accurately with probability $1-\delta$, with $n=O\left(\log\frac q{\eps\delta}\right)$ samples.\end{thm}
\begin{proof}First consider a single query. By Lemma \ref{priorconc}, at least a $1/2$ probability mass of the prior lies within $2^{-m/2}=2^{-n/2}$ of the prior mean along that query direction. We will first use this fact to show that the prior mean is not near a boundary of the resulting smart partition.

By Lemma \ref{safepart} with $\eta=\eps/24$, less than a total of $\frac14$ of the prior probability weight lies within $\eps/24$ on either side of all boundaries. Therefore, those regions cannot hold all of the probability mass within $2^{-n/2}$ of the prior mean, which implies that the prior mean is at least $\eps/24-2^{-n/2}$ away from the nearest boundary.

By Lemma \ref{priorconc} again\footnote{Technically, Lemma \ref{priorconc} only applied to the true answer's deviation from the prior mean. But the posterior mean consists of an average of several true answers, which will only concentrate tighter towards the prior mean. This is immediately clear when it comes to variance-based arguments like the proof of Lemma \ref{priorconc}, since $\Var(\frac{a+b}2)\le\frac12(\Var(a)+\Var(b))$.}, the probability of the posterior mean falling into a different interval than the prior mean is less than $2^{-n}/(\eps/24-2^{-n/2})^2=(\eps/(24\cdot2^{-n/2})-1)^{-2}$. If $2^{-n/2}\le\frac{\eps\sqrt\delta}{48\sqrt q}$, then, the probability of a different answer than the prior mean is at most $\frac1{\left(2\sqrt{q/\delta}-1\right)^2}<\frac\delta{2q}$.

Therefore, the smart rounded posterior mean answers the same as the Smart Rounded Prior Mean with probability at least $1-\delta/2q$. Since the Smart Rounded Prior Mean doesn't depend on the data, it can be simulated by the analyst, and its answers provide no additional information. Therefore, if the answers match, the analyst only learns that an event with probability $\delta/2q$ did not take place. Repeating for all $q$ queries, all answers match with probability at most $\delta/2$.

The Smart Rounded Prior Mean is also $\eps$-accurate: By Lemma \ref{priorconc} one last time, the prior mean is within $\eps/2$ of the true answer with probability $1-2^{-n}/\eps^2\ge1-\frac\delta{48^2q}>1-\frac\delta{2q}$. Moreover, the smart rounding answers with the midpoint of an interval of width less than $\eps$ containing the prior mean, so it is only at most $\eps/2$ off. Therefore, the smart rounded posterior mean is within $\eps$ of the true answer with probability $1-\delta$ over all $q$ queries, as desired.\end{proof}

Like with noise, we've demonstrated accuracy of rounding by comparing with the prior mean, which doesn't depend on the data. But unlike the prior mean, these algorithms also behave well on $LC_m$ for $n>m$, which is when there is enough data for the posterior to nail down the true hypothesis. In this case, it would still take $\exp(\Theta(m))$ queries for the analyst to learn this hypothesis, but this can be far less than $\exp(\Theta(n))$.

The next natural question would be whether smart rounding can succeed against polynomial classification. Unfortunately, that answer isn't immediately clear. We need an even smaller fraction of the hypotheses to lie outside of the same region around the posterior mean, and the corresponding concentration bounds from the $(2^k-1)$-wise independence of degree $\le k$ polynomials aren't strong enough. I'd conjecture that it does work, but this also isn't the end of the story; there is yet another difficult model.

\subsection{General Error-Correcting Codes}

From another perspective, the key feature of polynomial classification we've used is that it forms a linear error-correcting code with many codewords and large distance between codewords (polynomials). In fact, we can do this for any binary linear error-correcting code:
\begin{defn}Let $\C\subset\F_2^m$ be a linear error-correcting code of length $m$ over finite field $\F_2$. Define the model $\M_\C$ over universe $[m]\times\F_2$ to have the following prior: Each codeword $C\in\C$ corresponds to a hypothesis $h_C$ which is a distribution with weight $\frac1m$ on $(i,C_i)$ for each $i\in[m]$.\end{defn}
Recall that $\C\subset\F_2^m$ is said to have \emph{dimension} $k$ if $\abs\C=2^k$ and \emph{distance} $d$ if for any two distinct codewords $C,C'\in\C$, $C_i\neq C_i'$ for $d$ values of $i\in[m]$. Finally, recall that $\C$ is \emph{linear} if $\C$ is a linear subspace of the vector space $\F_2^m$.
\begin{lem}Suppose $\C\subset\F_2^m$ is a linear code with length $m$, dimension $k$ and distance $d$. Then there exists some $k-1\le n\le\frac{2mk}d$ such that if a curator receives $n$ samples from $\M_\C$, with probability at least $\frac d{4mk}$, the curator's posterior will place equal weight on exactly 2 hypotheses who differ on a subset of weight at least $d/m$.\label{generalcode}\end{lem}
\begin{proof}The curator begins with $2^k$ eligible hypotheses, and since $\C$ is linear, each new observation $(i,C_i)$ is either consistent with all of the eligible hypotheses or a $\frac1q$ fraction of them. Since the false hypotheses all have distance at least $d$ from the true hypothesis, there is a probability of at least $\frac dm$ that each sample falls into the latter category. If the curator receives and updates on samples one at a time, there will be some number of samples $n$ after which exactly $q$ eligible hypotheses remain. By Markov's inequality, with probability at least $\frac12$, this occurs before $n\le\frac{2mk}d$. Therefore, there is some $k-1\le n\le\frac{2mk}d$ such that this occurs at precisely $n$ data points with probability at least $\frac d{4mk}$.\end{proof}

Reed-Muller codes, which we've been calling polynomial classification, are linear binary codes with length $2^m$, rate $M=\binom m{\le k}$, and distance $2^{m-k}$. We can improve results like Theorem \ref{pcnpm} by instead using another code. For instance, the Justesen code (see \cite{justesen1972class}) is a linear code over $\F_2$ with rate and distance proportional to $m$; we can take for instance $k=m/4$ and $d=m/10$ (for large enough $m$). Then by Lemma \ref{generalcode} with $\C$ a Justesen code, there is some $\frac m4\le n\le5m$ such that with probability $\ge\frac1{10m}$, after receiving $n$ samples from $\M_\C$, the curator's posterior has two eligible hypotheses left which differ on a subset with weight at least $1/10$. Call such a model $\M_\C=J_m$ for simplicity.

The main advantage of Justesen codes relative to Reed-Muller codes is their constant rate, which implies a smaller universe. This naturally reduces the number of queries needed in the attack against noisy posterior mean of Theorem \ref{pcnpm}. But rather than stopping with noisy posterior mean again, we will next proceed to generalize this attack to handle other attempts to obfuscate the posterior mean, such as rounding.

\section{Extracting Information from Obfuscation}\label{extracting}

The big idea with the difficult model we'll now construct is this: We want to use approximately the same attack as the function in \eqref{slightdiagonalquery}, but where $B$ is now a random variable, independent of the current model and approximately uniformly distributed in $[0,1]$. If we can do this, there will be a probability around $1/\abs\Y$ that it will lie near one of the boundaries drawn, and the rounded posterior mean will leak information.

To complete this construction, we need to describe how to include independent random variables into the problem, how to construct a posterior mean that is nearly uniformly distributed in $[0,1]$, and how to aggregate the information leaked by the rounded posterior mean. We take each of those in turn.

\subsection{Independent Variables: Tensor Products of Models}

Given models $\M_1=(\X_1,\P_1)$ and $\M_2=(\X_2,\P_2)$, we say that a sample from the \emph{tensor product} $\M_1\otimes\M_2$ is an ordered pair of independent samples from each of the two models. In other words, the combined prior $\P_1\otimes\P_2$ consists of independently sampling distributions $\vp_i$ from each individual prior $\P_i$ and taking the product distribution $\vp_1\times \vp_2$ on universe $\X_1\times\X_2$.

Of course, since these samples are independent, the resulting posterior is just the tensor product of the posteriors from each of the models. The queries, as functions on the product space, however, can be more complicated than combinations of queries on the individual models. If the query function $f:X_1\times X_2\to[0,1]$ does not depend on its second argument, it amounts to a query from the first model, so the analyst can still ask all the questions she can ask in the original two models, and then some. In this way, tensor products are naturally never easier, and potentially harder, for the curator.

\begin{ex}The example with $q$ coins in the introduction is a $q$-fold tensor product of $\eps$-biased coins (abbreviated $BC_\eps$), which we will call a $q$th \emph{tensor power} and write as $(BC_\eps)^{\otimes q}$. As before, the $k$th query function is $f_k(x_1,\dotsc,x_q)=1$ if $x_k=H$ and $f_k(x_1,\dotsc,x_q)=0$ if $x_k=T$. These functions only depend on one argument, though, so they are functionally identical to querying those submodels, individual coins.\end{ex}

This example actually generalizes in some sense, which we formulate now. We can use tensor powers to eliminate the dependence on the error probability $\delta$ when it comes to lower bounds:

\begin{prop}Suppose that under a model $\M$ and for some $\eps,\delta>0$ (assume for simplicity that $1/\delta\in\N$), if the curator receives $n$ samples, the analyst can ask a series of $q$ queries, at least one of which the curator will answer $\eps$-incorrectly with probability at least $\delta$. Then under the tensor power model $\M^{\otimes1/\delta}$, if the curator receives $n$ samples, the analyst can ask a series of $q/\delta$ queries, at least one of which the curator will answer $\eps$-incorrectly with probability at least $1-1/e$.\label{deltadependence}\end{prop}
\begin{rmk}If we let $n_\M(q,\eps,\delta)$ be the maximum number of samples such that the hypothesis of the theorem holds, this shows that
\[n_{\M^{1/\delta}}(q/\delta,\eps,1-1/e)\ge n_\M(q,\eps,\delta).\]
We have been interested in $\sup_\M n_\M(q,\eps,\delta)$, which must then be at least $\sup_\M n_\M(q/\delta,\eps,1-1/e)$, or equivalently, $\Theta(\sup_\M(q/\delta,\eps,c))$ for any constant $0<c<1$.\end{rmk}
\begin{proof}The analyst strategy on $\M^{1/\delta}$ is simple: Repeat the same $q$ queries from her strategy on $\M$, in turn on each of the $1/\delta$ independent copies of $\M$. Each one will cause at least one answer to be $\eps$-inaccurate independently with probability at least $\delta$, so the overall failure probability is at least $1-(1-\delta)^{1/\delta}>1-1/e$.\end{proof}

This suggests that we can think of the (inverse) failure probability and the number of queries in approximately the same sense, at least when it comes to lower bounds.

\subsection{Powers and Tensor Powers of Models}

We will need one more technical construction to make a simple model whose posterior mean on a query is nearly uniformly distributed, although it isn't nearly as obvious why this will be necessary. Anyways, for a given model $\M=(\X,\P)$, we say that a sample from the \emph{$r$th power of $\M$}, which we will write as $\M^r$, is an $r$-tuple of samples from the \emph{same} hypothesis $\vp\sim\P$. Effectively, this means that the curator simply gets $rn$ samples rather than $n$.

Of course, $\M^r$ would seem at first glance to be strictly easier than $\M$ for the curator, because he gets more samples. It isn't quite that simple, though, since $\M^r$ also allows for more complicated queries, like testing for the $r$ samples to be distinct. In general, the analyst can now query appropriately bounded degree $r$ polynomials in the coordinates of $\vp$, rather than just appropriately bounded linear combinations. While we won't be using this, it's helpful to establish correct intuitions.

To be clear on the difference between the power and tensor power, both $\M^r$ and $\M^{\otimes r}$ generate samples that are $r$-tuples with coordinates drawn from some hypothesis in the distribution, but for $\M^r$, all coordinates are drawn from the same distribution, while for $\M^{\otimes r}$, all coordinates are drawn from different samples of the distribution from the prior.

\subsection{Uniform Model}

To build a model with a nearly uniformly distributed posterior mean, we start with a prior that is uniformly distributed: The uniform model on universe $\{0,1,\dotsc,k-1\}$, which we denote $U_k$. This prior is simple: It is distributed uniformly over the simplex $\{\vp:p_i\ge0\forall i,\sum p_i=1\}$, with respect to the usual $(k-1)$-dimensional volume metric. We will only be using $k=2$, but the more general formulation is used in other natural problems.

The posterior mean on this metric is well-known: If the data shows $n_i$ copies of option $i$ out of a total of $n$ samples, the posterior mean is at
\[\left(\frac{n_0+1}{n+k},\frac{n_1+1}{n+k},\dotsc,\frac{n_{k-1}+1}{n+k}\right),\]
since the posterior in this case is a Dirichlet distribution. Moreover, all tuples $(n_0,n_1,\dotsc,n_{k-1})$ of nonnegative counts with sum $n$ are equally likely.

\subsection{Building a challenging model}

We now have all of the ingredients to construct a problematic model for all posterior mean-based approaches. Here it is in this notation:
\[J_m\otimes(U_2^8)^{\otimes(q-1)}.\]
That is, we take a tensor product of the quadratic classification problem and $q-1$ instances, each repeated 8 times, of a uniformly randomly biased coin.

Before examining each of these components, let's describe what the analyst does. Her attack will again consist of a series of $q-1$ exploratory queries followed by one final exploitative query. Each of the exploratory queries will use a fresh uniform random variable to probe the value of the function at a point $y_i$. The query function, motivated by \eqref{slightdiagonalquery}, is
\begin{align}
f_i:[m]\times\F_2\times(\{0,1\}^8)^{q-1}&\to[0,1]\\
(y,z,(x_{1,1},\dotsc,x_{1,8}),\dotsc,(x_{q-1,1},\dotsc,x_{q-1,8}))&\mapsto\begin{cases}z\dfrac{m-1}{4n+1}&\text{ if }y=y_i\\x_{i,1}&\text{ if }y\neq y_i.\end{cases}\label{uniformdiagonalquery}
\end{align}
Here, the $x_{i,1},\dotsc,x_{i,8}$ are the 8 copies of the $i$th uniform random variable. We aren't writing a more complicated polynomial of them, so our functions will only ever depend on one of them from each set. Conditional on $y\neq y_i$, the expectation of this function is simply the probability of coin $i$ landing on $1$.

The big idea is that these slightly correlated queries will seemingly be about the uniform random variables, but the relevant information to the analyst is in the slight adjustment depending on whether $\ell(y_i)=1$. The 8th power and that fraction in the query are much more technical components: They're designed to directly tune an observation of a single additional $2$ on a uniform random variable to a change in information about whether $f(y_i)=1$ for the posterior mean. Note that we take $n\ge m/4$ so that the fraction is less than 1.\footnote{Technically, we only had $n\ge m/4-1$, but the parameters in the Justesen code were not tight. We could easily instead have picked a Justesen code with rate at least $1/4+1/m$, but we keep it in this form for clarity.}

To demonstrate this tuning, let's understand that posterior mean on each of these queries. Suppose that out of the $8n$ samples of the $i$th coin, $s_i$ of them were 1's. If the curator has enough information to deduce that $\ell(y_i)=0$, the posterior mean is
\[0\cdot\frac1m+\frac{s_i+1}{8n+2}\left(1-\frac1m\right)=\frac{s_i+1}{8n+2}\cdot\frac{m-1}m=:a_{s_i}\]
for clarity. On the other hand, if the curator deduces that $\ell(y_i)=1$, the posterior mean is
\[\frac{m-1}{4n+1}\frac1m+\frac{s_i+1}{8n+2}\left(1-\frac1{2^m}\right)=\frac{s_i+3}{8n+2}\cdot\frac{m-1}m=a_{s_i+2}.\]
Finally, if the curator is uncertain on the value of $\ell(y_i)$, the posterior mean is the average of these values, or $\dfrac{s_i+2}{8n+2}\cdot\dfrac{m-1}m=a_{s_i+1}$.

Recall that for the uniform prior, the $s_i$ are uniformly distributed over the integers between $0$ and $8n$. Therefore, nearly every possible posterior mean value has the same probability under each of the three cases. The only exceptions are on the ends: $a_0$ is only possible if $\ell(y_i)=0$, $a_1$ if $\ell(y_i)=0$ or is unknown, $a_{8n+1}$ if $\ell(y_i)=1$ or is unknown, and $a_{8n+2}$ if $\ell(y_i)=1$. This shifting of some of the probability mass from one end of the interval $[0,1]$ to the other will be something we can detect by simple counts no matter what the mechanism does with the posterior mean.

We are now ready to show that this defeats a wide range of curator algorithms:
\begin{thm}\label{main}Suppose that the curator always outputs some possibly randomized function $g(f(\P),\E[f(\P')])$ of the prior and posterior mean on a query function $f$. Then under model $(J_m\otimes(U_2^8)^{\otimes(q-1)})^{\otimes m}$, for some $m=\Theta(n)$, there is an adaptive analyst strategy which causes the curator to answer constant-inaccurately with constant probability, using only $q=O(n^4\log n)$ queries.\end{thm}

The functional notation here indicates that the output is allowed to be a function of the prior distribution $f(\P)$ on the query and the posterior mean $\E[f(P')]$, as the smart rounding algorithm is.

\begin{proof}First consider one of the copies of $J_m\otimes(U_2^8)^{\otimes(q-1)}$. By Lemma \ref{generalcode} for Justesen codes, we can pick $m$ and $n$ satisfying $\frac m4\le n\le 5m$ so that with probability at least $\frac1{10m}=O\left(\frac1n\right)$, the curator's posterior will put equal weight on two hypotheses which disagree on at least $\frac1{10}$ of $[m]$.

In this case, the analyst must simply find out what the curator knows about that function and query one of the two remaining hypotheses, forcing the curator to answer $1/20$-inaccurately with probability at least $1/2$. As we've previewed, the analyst determines what the curator knows by asking $q-1$ queries given in \eqref{uniformdiagonalquery}. These all have the same distribution on the prior $f_i(\P)$ (which we hereafter omit from $g$ for clarity), so the curator's answers may only depend on the posterior mean $\E[f_i(\P')]$. Recall that this always takes on the values of $a_s$ for some value of $s\in\{0,1,\dotsc,8n+2\}$ depending on the uniform random variable and what the curator knows about $\ell(y_i)$.

Clearly the curator's function must have $g(a_0),g(a_1)<\frac12<g(a_{8n+1}),g(a_{8n+2})$ with near certainty or one of these answers will come into effect but be too far off. Therefore, the curator will give answers below $\frac12$ with a probability $O\left(\frac1n\right)$ higher if $\ell(y_i)=0$ and at least $O\left(\frac1n\right)$ lower if $\ell(y_i)=1$, as compared with the case where $\ell(y_i)$ is unknown.

Therefore, by taking $y_i$ to be the same point $y\in\Y$ for $O(n^2\log n)$ values of $i$, the analyst obtains estimates of this probability that are additively precise to within $O\left(\frac1n\right)$ on each of the counts of the potential values with error probability less than $\frac1n$. By comparing these results across all $y\in\Y$, the analyst can determine what the curator knows about the function using only $O(n^3\log n)$ queries, and in the final query, exploit this knowledge.

Finally, we can amplify this probability of at least $\frac1{10m}$ to a constant with Proposition \ref{deltadependence} by introducing the outer tensor power and paying an additional factor of $m=O(n)$ in the number of queries.\end{proof}

\begin{cor}Under the same model, the same adaptive analyst strategy causes the Bayesian private multiplicative weights curator to answer constant-inaccurately with constant probability using only $O(n^4\log n)$ queries.\label{pmwmonster}\end{cor}
\begin{proof}Recall that this curator first looks at the answer that a proxy distribution gives. The proxy distribution is initialized as the prior mean, a distribution which makes every uniform random variable equally likely to be $0$ or $1$. Since all updates to the posterior are functions of the previous queries, the proxy distribution will never reflect any knowledge of the random variables that have not previously played a role in queries. Therefore, the proxy distribution will answer the query \eqref{uniformdiagonalquery} with an answer within $\frac12\pm\frac1m$.

The curator releases this proxy answer if it is close enough to the true posterior mean. This definition of close enough is probabilistic, but the probability of a deviation on the order of $1/2$ is negligible. Therefore, if the posterior mean is any of $a_0,a_1,a_{rn+1},a_{rn+2}$, which are all far from $\frac12$, the curator will ignore the proxy answer and use the fallback method instead, which fell under the scope of Theorem \ref{main}.\footnote{We do make one slight adjustment to the analyst strategy: Instead of simply counting the number of answers above or below $\frac12$, we count the number above $\frac34$ or below $\frac14$. In this way, all of the answers we are counting are due to the fallback method and not the proxy distribution. Once again, we must still have $g(a_0),g(a_1)<\frac14$ and $g(a_{rn+1}),g(a_{rn+2})>\frac34$, so the difference in the counts will be noticeable depending on what the curator knows about $\ell(y)$.} After spending $O(n^4\log n)$ queries figuring out what the curator knows, the analyst exploits this knowledge in the usual way.\end{proof}

\section{Discussion}\label{disc}

These results are perhaps a bit surprising. The same analyst strategy of employing slightly correlated queries to leak information was able to defeat two different strong curator approaches, smart rounding and private multiplicative weights. Theorem \ref{main} is also quite general, showing a limitation to any attempt to obfuscate.

Notably, this is not merely a result of the focus of the Bayesian context on the posterior mean. While the parameters in the queries that the analyst asks are directly tuned to match the posterior mean, it is easy to see that the same could be done for the empirical mean, posterior median, or any other such aggregation.

The general attack here consists of two steps: (1) Use slightly correlated queries to learn what the curator knows, and (2) exploit that knowledge to find a query which he is unable to answer. Whatever method the curator used to represent the data originally, the analyst would be able to find it and ask the query the curator won't be able to answer accurately.

We now survey some features and proposed fixes to this problem. If your natural inclination is to ask, ``Well, what about X?'' then this is the section for you; otherwise, it can be skipped.

\subsection{Stability}

A key component of these classification problems is the moment when the curator is uncertain between two hypotheses which differ on a significant fraction of the space. In such a situation, with decent probability, the next data point will distinguish those hypotheses, moving the posterior mean by at least $\eps$ with respect to some query.

In other words, the posterior mean algorithm is not \emph{stable} on classification problems, since its answers to some queries change significantly when a single sample is added. This suggests that there might be promise in the perspective of another line of work, that of algorithmic stability. In 2002, Bousquet and Elisseeff \cite{bousquet2002stability} offered stability in a slightly different context as a conditon guaranteeing generalization, and both BNSSSU \cite{bassily2015algorithmic} and Hardt, Recht and Singer \cite{hardt2015train} have adapted this notion to adaptive data analysis.

However, mandating stability in this context is actually counterproductive. In classification, every sample point is either equally likely or impossible under each hypothesis, and therefore yields very strong information. As is, the posterior only puts nonzero weight on the eligible hypotheses, those that agree on every point. To achieve stability, we must soften this requirement, for instance by putting a weight proportional to $c^k$ on a hypothesis that makes $k$ errors, for some $0<c<1$.

Unfortunately, this is actually even more vulnerable to the same attack. Suppose the curator has enough data to give the true hypothesis between $1/4$ and $3/4$ of the total weight. Then the weight is high enough that the analyst's investigatory queries can isolate that hypothesis, but low enough that querying it directly will surely be incorrect.

Moreover, attempting to introduce stability makes this intermediate state last even longer than it did previously! By only reducing incorrect hypotheses' weights by a factor of $c>0$, the mass on the true hypothesis can only spend more time in the range $[1/4,3/4]$. Thus, stable versions of the posterior mean algorithm actually behave poorly on an even wider range of parameters than the posterior mean itself.

\subsection{Decomposition}\label{decomposition}

This problem is hard for the usual methods, but it's actually quite easy to match the static bound if we allow ourselves to break it open. With a tensor product model like this one, we can distinguish between data from the two component models, and ignore the data from $J_m$ (until there's enough to uniquely determine the function) while fully updating based on the data from $(U_2^8)^{\otimes(q-1)}$. Since each of those behave well on their own, this ``decomposition'' curator strategy works on this problem in particular.

Unfortunately, it's difficult if not impossible to turn this specific strategy hack into a general algorithm. For instance, if we introduce tiny perturbations to the independence of these two components, it becomes hard to separate our knowledge of each, or even to define the two components in the first place.

\subsection{Adaptivity Detection}\label{detect}

Another class of algorithms that does well, but in a fragile way, against this particular problem aims to detect when the analyst is using adaptivity and utilizes previously unseen data whenever it does.

The first such algorithm in the literature, DFHPRR's EffectiveRounds, reserves about 99\% of its data for a ``checker'' and splits the remaining points into $r$ equal-sized ``estimation samples.'' The algorithm proceeds with answers from the first estimation sample until those answers deviate from the checker by more than a noisy threshold, after which the sample is discarded. In this way, if there are only $r$ rounds of adaptivity, the algorithm only needs $O(r)$ times as much data. (To make this better than simple sample splitting, this works even if the researcher doesn't know when the rounds of adaptivty occur.)

The explore-exploit attack we've described only uses two rounds of adaptivity, so EffectiveRounds would defeat it with only a constant increase in the sample size. However, the tensor product construction allows us simply to repeat the same problem $n$ times, and so defeat this algorithm using only $\tilde O(n^5)$ queries.

Could we apply a similar philosophy to do better in a Bayesian context? Perhaps. One note of hope: When the Justesen code problem posterior is uncertain between two hypotheses, only a constant number of additional data points are needed to resolve this uncertainty, in expectation. So if we get stuck in a problem like this, we only need to look at a few more data points to break out. One could imagine a curator algorithm that keeps some number of data points in reserve and updates on them only if the posterior is not concentrated on some query.

Unfortunately, looking at even one extra data point is too many. By a similar trick, the subsequent copies of the Justesen classification problem could involve a slightly higher parameters $m$, causing the curator to look at at least one new data point every $\tilde O(n^4)$ queries. Repeat this $n$ times, and the curator runs out of points to add after only $\tilde O(n^5)$ queries. Again, this isn't elegant, and there could potentially be a viable curator algorithm along these lines, but it won't be trivial.

\subsection{Restricting Models}

Perhaps the general problem that we've posed is too difficult, and we should restrict the models somehow, in addition to ruling out use of asymmetric information. There seem to be two general approaches in the literature to doing this, bottom-up and top-down.

The bottom-up approaches are the most promising, but the most restrictive. In general, they start with a particular use case of adaptive data analysis and build algorithms to handle those in particular, hoping to again to achieve tight sample complexity results in a specific setting.

For instance, Russo and Zou \cite{russo2015controlling} focus on the situation where researchers compute a variety of statistics and report only the ``best'' one or several, such as the smallest $p$-values. In this setting, they are able to control the bias by bounding the mutual information between the choice of statistic(s) and all of the statistics' realized values.

Similarly, Blum and Hardt \cite{blum2015ladder} examine the situation of releasing an approximately correct machine learning competition leaderboard, a common and obvious source of overfitting to a frequently used holdout set. With this objective, the leaderboard algorithm can avoid releasing scores unless a new submission is a significant amount better than the previous best, which effectively limits the information leakage.

Alternatively, a top-down approach hopes to solve the general problem under some restricted conditions that preclude examples like the ones introduced in this paper. The trickiest step here is formulating what the restriction should be.

One possibility is to restrict the dimension of the problem. Indeed, the universe of $J_m\otimes(U_2^8)^{\otimes q-1}$ has dimension $\log\abs\X=\log m+1+8(q-1)$. While big data often deals with situations where the dimension is greater than the number of data points $n$, it might not be as high as $n^4$ (our number of queries).

This is essentially the technique that the guarantees for the original Private Multiplicative Weights \cite{hardt2010multiplicative} algorithm utilize, although their results are written with a multiplicative factor of the dimension (to some power) in the required sample complexity.

While this does yield effective algorithms, it isn't a priori clear why the dimension should play a role in the first place. After all, every query projects the space of distributions $\Delta(\X)$ down to a single dimension, fitting it within $[0,1]$, so the queries themselves are dimension-independent.

\subsubsection{Controlled Trials}

We have one interesting new restriction to offer that might make the problem easier. To understand the motivation for this restriction, consider a common setup in scientific analysis: Randomized controlled trials, also known as A/B testing in business.

In this framework, samples are drawn independently from a population and assigned to one of two groups. One of the groups receives a change of some kind while the other group stays the same or receives an ineffective version of the change (placebo) if appropriate. The study aims to compare some output variables on each group, possibly restricted to subpopulations of the original sample. This type of problem is ripe for high generalization error, because of the range of questions to ask and the opportunity to ask them adaptively.

To formulate this in the same sort of mathematical framework, we consider the universe to be a product $\X\times\Y\times\Z$. Here, $\X$ will capture the demographic data about the population, and we will assume the marginal over $\X$ is known. $\Y$ then captures the group data: Whether the sample was placed in the experimental group or the control group. For simplicity, we will assume there are just two groups, so $\Y=\{0,1\}$, and the samples are uniformly randomly assigned to each group, independent of their demographic data. Finally, $\Z$ is the output variable in question. For even greater simplicity, we will also take $\Z=\{0,1\}$, corresponding to a single binary variable being studied, such as whether someone recovered from an illness or died.

Since the object of study is the difference between the two groups, we will require that all queries be of a specific form:
\begin{equation}\E_{(x,y,z)\sim\vec p}[z1_{x\in S}|y=1]-\E_{(x,y,z)\sim\vec p}[z1_{x\in S}|y=0]\label{rctquery}\end{equation}
for some subset $S\subset\X$. In other words, this is measuring the difference in the probability of the outcome variable being $1$ between the two groups on a specific subpopulation, scaled by the (publicly known) fractional size of the subpopulation.\footnote{This scaling matches that of statistical queries. It also seems appropriate, as opposed to say $\E_{(x,y,z)\sim\vec p}[z|y=1,x\in S]$, because when $S$ is small, we would need more data to estimate this probability accurately.} Now, up to a factor of two and a shift, this is the same as $\Pr[y=z\text{ and }x\in S]$. In other words, for the purposes of queries of the form in \eqref{rctquery}, we can collapse $\Y\times\Z=\{0,1\}^2$ down to two points, identifying the pairs $(0,0)$ and $(1,1)$, and $(0,1)$ and $(1,0)$. Let $w$ be the indicator function on the event $y=z$, so we are asked to estimate $\E[w1_{x\in S}]$.

Moreover, we can consider this as an average value of the function $f:\X\to[0,1],x\mapsto\E[w|x]$ corresponding to the difference in the effect of treatment between the groups, rescaled to fall within $[0,1]$ with $f(x)=1/2$ meaning no effect. Therefore, the queries amount to asking the agreement between the true probability function $g:\X\to[0,1]$ and indicator functions $f(x)=1_{x\in S}$. In other words, these are merely a slight relaxation of classification problems! There is one small change: Instead of measuring agreement with functions to the set $\{0,1\}$, we are measuring agreement with functions to the interval $[0,1]$, where we say that the value $t\in[0,1]$ corresponds to agreement $t$ with 1 and $1-t$ with 0. This is just a convex relaxation of the original classification problem framework we introduced section \ref{classification}.

Restricting to (soft) classification problems might not seem promising since most of our problematic examples come from that framework. However, the tensor product, an integral construction of the final challenging model that created slightly correlated queries, cannot be expressed in a classification model. This gives hope that some version of the posterior mean might solve all classification problems, and therefore provide a framework for avoiding overfitting from multiple comparisons in parallel group randomized controlled trials.

\subsection{The Prior as an Analysis Tool}

The main message of introducing Bayesian adaptive data analysis is not to propose that every problem in usual (frequentist) data analysis be translated into the Bayesian setting by writing down a prior and then doing appropriate posterior inference however long it takes. While there are certainly scenarios where a Bayesian approach is valid, we recognize that modeling all of one's prior beliefs is in general quite difficult.

Instead, we were originally motivated to add the Bayesian prior to this problem to avoid the problematic lower bounds which involved an analyst who already knew the distribution exactly. It is also helpful to see the prior as an analysis tool to distinguish which lower bound considerations are difficult merely because of information asymmetry in the original problem, and which are still applicable in contexts outside of that.

As we've discussed, all of the algorithms that we have described here in fact have very natural analogues in the frequentist setting: The posterior mean compares to the empirical mean, the prior mean can serve as a less arbitrary initial proxy distribution in the private multiplicative weights algorithm, and we can add noise to or round answers produced by either Bayesian or frequentist algorithms.\footnote{While the prior-based smart rounding we introduced does depend on observing the prior, we only actually utilized this to prove that the prior mean could be chosen to be far from a rounding boundary.} By demonstrating the success or lack thereof of an algorithm in the Bayesian setting, the moral implication is that we should treat the frequentist analogue accordingly in practice.

We say ``moral implication'' because the corresponding statements in the frequentist setting are not well defined. Without introducing a prior, we have no rigorous definition of when an analyst is taking advantage of side information. It is only when passing to the Bayesian analogues that this distinction becomes apparent.

In this light, our main result, Theorem \ref{main}, isn't just saying that these specific adjustments to the posterior mean fall short of matching the static bound. Instead, it indicates a new type of limitation to the general program of obfuscating the data that isn't due simply to an information asymmetry: \textbf{Standard methods for obfuscation (even more sophisticated ones) cannot avoid leaking information against slightly correlated queries.} While this result shows a clear separation from the static case, the fact that this only gives a polylogarithmic rather than a polynomial dependence of $n$ on $q$ suggests (in an even less certain sense) that this limitation is not nearly as constraining as the attacks based on information asymmetry.

Taking a step back, this situation is actually somewhat analogous to the field of solar astronomy: When there is a solar eclipse, researchers can study the sun's corona with otherwise overpowering light from the sun itself blocked. In the same way, Bayesian adaptive data analysis blocks the otherwise dominant class of lower bound techniques, those exploiting information asymmetry. With those removed, we can better understand other sources of error in adaptive data analysis as commonly carried out in practice.

\section{Conclusions}

What makes adaptive data analysis inherently hard? What stops the curator from answering as many adaptive queries as he can static queries?

The picture from previous lower bounds was bleak: Powerful analysts with complete knowledge of the distributions they were pretending to try to study could compile errors and catch the curator making a mistake. Having already somehow gotten access to the distribution itself, these analysts just needed $O(n^2)$ queries to crack the curator's $n$ samples are and stump him with a query about the rest of the distribution he hadn't seen. The proposed solutions therefore naturally sought strong differential privacy techniques to protect every data point from the gaze of these nearly omniscient analysts.

For practitioners, though, this worry about superintelligences disguising themselves as curious seekers of truth seems over the top. At the very least, surely this can't be the only difficulty with adaptivity! To try to understand what else makes adaptivity difficult, we translated the problem of DFHPRR over to the Bayesian setting, which we then set out to explore. By ruling out the biggest concern in the previous picture, studying Bayesian adaptive data analysis allows us to learn what other sorts of problems might arise.

Initial scouting reports are mixed. While seemingly less dangerous for curators, the Bayesian setting still has difficult problems and new types of tricky analysts to be concerned about. The obfuscation techniques from the original problem do prove useful in the Bayesian setting as well, but still fall far short. In particular, in Theorem \ref{main}, we constructed one problem on which a carefully probing analyst needs only $\tilde O(n^4)$ queries to cause the curator to make a mistake, against the entire family of previously proposed curator algorithms. This example illustrates a second difficulty in adaptive data analysis: Slightly correlated queries can leak information past the usual obfuscation techniques, which can be problematic if posterior uncertainty becomes concentrated in a single direction.

Future work in Bayesian adaptive data analysis can further explore this world, helping to find and clarify the realistic problems with adaptive data analysis.

\section{Acknowledgments}

I would like to thank Jon Kelner, Jerry Li, Adam Sealfon, and Thomas Steinke for numerous helpful conversations for the entire duration of this work. I also received financial support from the United States Department of Defense (DoD) through the National Defense Science and Engineering Graduate Fellowship (NDSEG) Program.

\bibliography{sources}
\bibliographystyle{plain}

\section{Appendix A: Noisy Posterior Mean on Linear Classification}

In this appendix, we prove Theorem \ref{npmlc}, showing that the noisy posterior mean algorithm answers accurately on the linear classification problem.
\begin{proof}First, if $n>2m$, with probability $1-2^{-\Omega(n)}$, the curator will have enough data to completely determine the hypothesis (it just takes $m+1$ novel points, and each point is novel with probability $>1/2$). In that case, the error of the noisy posterior will simply be the Gaussian noise added, which is designed to be of size at most $\eps$ with probability at least $1-\frac\delta q$, as desired.

So we may suppose $n\le 2m$. The statement then amounts to showing that the analyst answers correctly on every query, except with probability $q\exp(-\Omega(m\eps^2))>q\cdot2^{-\Omega(m)}$.

Following the sketch in Section \ref{npm}, we wish to understand the analyst's knowledge of the distribution via her posterior. Just like the curator, the analyst's prior in linear classification is uniform on all $2^{m+1}$ possible linear functions. We will inductively prove the following carefully-calibrated claims:

\begin{lem}Let $\delta'=2qm2^{-m/2}$. After $k$ queries for $k\le q\le1/3\delta'$, with probability $1-O(k^2\delta')$, (i) the curator's answers on all queries are correct, and (ii) the analyst's posterior puts weights within $\exp(\pm k/q)/2^{m+1}$ on each hypothesis, except for a $k\delta'^2$ fraction of the hypotheses, which themselves have at most a total weight of $k\delta'$.\end{lem}
\begin{proof}We show this by induction on $k$. From the analyst's perspective, to answer a query, a random hypothesis according to the posterior is chosen, the curator receives $n$ data points from that posterior, answers the query according to the data, and then all of the hypothesis weights are adjusted accordingly. For simplicity, we will actually assume that the curator answers according to a single hypothesis, and argue that the curator's answers are only more true if he answers with the average of several hypotheses.

First, we label as ``$k$-bad'' all hypotheses that give answers more than $\eps'=\frac1{4qm}\ll\eps$ away from the mean on the $k$th query, and as ``$k$-good'' any hypothesis that is not $l$-bad for any $l\le k$. We already know there won't be many $k$-bad hypotheses: By Lemma \ref{priorconc}, a $\frac{2^{-m}}{4\eps'^2}=4q^2m^22^{-m}=\delta'^2$-fraction of the hypotheses are $k$-bad, so all but a $k\delta'^2$ fraction of the hypotheses will be $k$-good, as desired. We now must show that the $k$-good hypotheses maintain approximately similar weights, while bad hypotheses have small total weight, with high probability.

We then condition on the randomly chosen hypothesis (in the analyst's model of how the curator works) being $(k-1)$-good, which adds $(k-1)\delta'$ to the error probability by the inductive hypothesis. Then we condition on it not being $k$-bad, which adds at most $\exp((k-1)/q)\delta'^2<4\delta'^2<\delta'$ to the error probability, again by the induction hypothesis's upper bound on all of the $(k-1)$-good hypotheses' weights. In all, this adds only $k\delta'$ to the error probability, which is fine since we have $O(k^2-(k-1)^2)\delta'=O(k)\delta'$ extra error probability to work with. So we may assume that the true hypothesis's answer $a$ is within $\eps'$ of the mean $\mu$ on the query. Moreover, with the same error probability, the curator's estimate is also within $\eps'$ of the mean $\mu$ on the query,\footnote{In the real scenario where the curator answers according to multiple hypotheses, we have a little more work to do. Digging into the variance-based proof of Lemma \ref{priorconc} would show that averages of hypotheses are even more likely to fall close to the mean than individual hypotheses, and therefore, the unnoised answers of the curator will be $k$-good.} so the curator's unnoised answer is $2\eps'<1/qn<1/2\sqrt n<\eps/2$-accurate. Since the noise is calibrated to be less than $\eps/2$ with probability $1-\delta/2q$, this means that the curator's answers are accurate with the desired probability, satisfying (i).

The analyst then updates her beliefs according to this noised answer. Since the curator adds Gaussian noise with variance $1/4n>1/2m$, the likelihood density of answering with $a'$ is proportional to $\exp(-m(a-a')^2)$, where $a$ is the answer of the true hypothesis.\footnote{The coefficient of $m$ in this expression is possibly lower, but this will only make the weights more similar to each other.} By Bayes' rule, this is the factor we multiply the weights on hypotheses with answer $a$, before renormalizing. If $a'$ is outside of the range $[\mu-\eps',\mu+\eps']$, the largest and smallest weights to good hypotheses will occur on the ends of the interval, so the likelihood ratio of any pair of good hypotheses is bounded by
\[\exp\abs{m(\mu+\eps'-a')^2-m(\mu-\eps'-a')^2}=\exp(4m\eps'\abs{\mu-a'})\le\exp(2m\eps')=\exp(1/2q),\]
where we have simply bounded $\abs{mu-a'}<\eps<1/2$ with very high probability.\footnote{For simplicity, we are analyzing a version of the noisy posterior mean algorithm that does not truncate all answers to fall within $[0,1]$. However, it is clear that such truncation only reduces the information that the analyst receives.} If $a'\in[\mu-\eps',\mu+\eps']$, the smallest weights, near 1, will fall on any hypotheses with answers close to $a'$, so the likelihood ratio will be bounded by the weight at an endpoint: $\exp(m(2\eps')^2)<\exp(2m\eps')=\exp(1/2q)$ as well.

Therefore, all of the good hypotheses maintain approximately the same weights as each other: No pair can have a likelihood ratio greater than $\exp(1/2q)$. To show that this means all good hypotheses indeed have weights near their original values, we must bound the total weight on bad hypotheses.

The argument here is a little different. When the new $l$-bad (i.e. also $l-1$-good) hypotheses are determined by the choice of the $l$th query, they collectively have weight at most $\le\exp(l/q)\delta'^2<4\delta'^2$. Now, since this is a correct posterior, those hypothesis weights are also equal to the expectation of their weights after any future data. So by Markov's inequality, after the $k$th query, the new $l$-bad hypotheses have weight at most $\delta'$ with probability $1-4\delta'$, and collectively, by a union bound, all $l$-bad hypotheses for $l\le k$ have weight at most $k\delta'$ with probability $1-4k\delta'$. Since we again have error probability $O(k\delta')$ to spare, we have proved the last claim of the lemma.

Finally, since the bad hypotheses have total weight at most $k\delta'$, the average weight on the $k$-good hypotheses is between $\frac1{1-k\delta'^2}$ and $1-k\delta'$ times their original weight ($2^{-m-1}$). But $1-k\delta'^2>1-k\delta'\ge1-k/3q>\exp(-k/2q)$, using the assumed bound on $q$, so the average weight of the $k$-good hypotheses is within $\exp(\pm k/2q)$ of their original weight. Since no pair of $k$-good hypotheses has weight ratio exceeding $\exp(k/2q)$, this means that every $k$-good hypothesis has a weight within $\exp(\pm k/q)$, as desired.\end{proof}

Taking $k=q$, then, the curator is $\eps$-accurate with probability at least $1-O(q^2\delta')=1-O(q^3m2^{-m/2})$. (Clearly this is meaningless unless $q<1/3\delta'$.) Since $m\le O(2^{m/6})$, we can also write this as $1-O(q^32^{-m/3})$. Therefore, when $q\le O(2^{m/9}\delta^{1/3})$, the error probability is at most $\delta$. Translating this in terms of $n\le2m$, we have shown that the curator wins if $n>18\log q+6\log\Omega(1/\delta)=\Omega(\log q/\delta)$. Therefore, if $n\ge\Omega\left(\frac1{\eps^2}\log\frac q\delta\right)$, the noise is small and the curator is accurate, as desired.\end{proof}

\section{Appendix B: Constructing Smart Partitions}

In this appendix, we prove Lemma \ref{safepart}, showing that smart partitions exist.
\begin{proof}First, we show it suffices to prove a single point version of this, which is simple enough that it may be of independent interest:
\begin{lem}\label{safepoint}For any distribution $D$ on $[0,1]$, there exists a point $x\in(0,1)$ such that $\forall\eta>0$,
\[\Pr_D[[x-\eta,x]],\Pr_D[[x,x+\eta]]\le2\eta.\]\end{lem}
\begin{rmk}This lemma is tight: If $D$ is identically equal to $1/2$, then taking $\eta=\abs{\frac12-x}$ makes one of the terms on the left equal to 1 and $2\eta\le1$.\end{rmk}
For $i=1,\dotsc,m-1$, we will pick $x_i\in I_i\left(\frac{2i}3\eps,\frac{2i+1}3\eps\right)$, so $x_{i+1}-x_i>\left(\frac{2i+2}3-\frac{2i+1}3\right)\eps=\frac\eps3$ and $x_{i+1}-x_i<\left(\frac{2i+3}3-\frac{2i}3\right)\eps=\eps$. Define another distribution $D'$ as follows: For each $i$, move all of the weight from $\left[\frac{4i-1}6\eps,\frac{2i}3\eps\right]$ to the point $\frac{2i}3\eps$, and all of the weight from $\left[\frac{2i+1}3,\frac{4i+3}6\eps\right)$ to $\frac{2i+1}3\eps$. Since the statement of Lemma \ref{safepart} is trivial for $\eta\ge\frac\eps6$, weight in those intervals is too far from $x_{i-1}$ or $x_{i+1}$ to be included in either term, so only its distance from $x_i$ matters. Moreover, we've decreased that distance, so it suffices to prove the result on $D'$ instead of $D$, under the constraint that each $x_i\in I_i$.

Assuming that $D'$ puts positive weight in $I_i$ (otherwise we can ignore that interval), consider $D'$ restricted to $I_i$, and rescaled by a factor of $3/\eps$ to become an interval on $[0,1]$; call this new distribution $D_i'$. By Lemma \ref{safepoint}, there exists some $x'\in[0,1]$ such that $\forall\eta>0$,
\begin{align*}
\Pr_{D_i'}[(x'-\eta,x')],\Pr_{D_i'}[(x',x'+\eta)]&<2\eta\\
\Pr_{D'}[(x-\eta\eps/3,x)],\Pr_{D'}[(x,x+\eta\eps/3)]&<2\eta\Pr_{D'}[x\in I_i]\\
\Pr_{D'}[(x-\eta,x)],\Pr_{D'}[(x,x+\eta)]&<\frac{6\eta}\eps\Pr_D[x\in I_i]\\
\Pr_{D'}\left[\bigcup_{i=1}^{m-1}(x_i-\eta,x_i)\right],\Pr_{D'}\left[\bigcup_{i=1}^{m-1}(x_i,x_i+\eta)\right]&<\frac{6\eta}\eps.
\end{align*}
where $x'$ is the image of $x$ under the rescaling, and in the final line we have used the fact that the intervals $I_i$ are disjoint. Having proved the result for $D'$, we have shown that Lemma \ref{safepoint} implies Lemma \ref{safepart}, apart from the algorithmic statement.\end{proof}
\begin{proof}[Proof of Lemma \ref{safepoint}]By the density of continuous functions, we may assume that $D$ has a continuous cumulative density function (no point masses). Then we can write $\Pr_D[(a,b)]=\Pr_D[[a,b]]=:\Pr[a,b]$ for short.

Suppose the statement is false, so for all $x\in[0,1]$ there exists $\eta$ such that either $\Pr[x-\eta,x]>2\eta$ or $\Pr[x,x+\eta]>2\eta$. We will construct a series of points which will show a violation of this claim.

We start at $x=0$, where we can rule out one of the possibilities, since $\Pr[-\eta,0]=0$. We will first construct a series of points $0=x_0<y_0<x_1<y_1<\dotsb$ inductively such that (i) $\Pr[0,x_k]\ge x_k$, (ii) $\Pr[0,x_k]-2x_k$ is strictly decreasing in $k$, and for all $\eta>0$, (iii) $\Pr[x_k-\eta,x_k]<2\eta$ and (iv) $\Pr[y_k,y_k+\eta]<2\eta$. To visualize one step of this, see Figure \ref{cdf}.

\begin{figure}
\begin{center}\begin{tikzpicture}[scale=1.5]
\path [fill=red!20] (3,5) -- (3.5,6) -- (3,6) -- cycle;
\path [fill=red!20] (4.3,5.6) -- (4.3,1) -- (2,1) -- cycle;
\draw (1,1) rectangle (6,6);
\draw (2,3) -- (3.5,6);
\draw (3,3) -- (4.5,6);
\draw [dashed] (2,3) -- (3,3) -- (3,5);
\draw [dashed] (4.3,3) -- (4.3,5.6);
\draw [very thick] (2,3) to [out=30,in=220] (3,5) to [out=40,in=190] (4.3,5.6);
\draw [dotted] (2,3) -- (5,6);
\node [below] at (2,3) {$x_k$};
\node [below] at (3,3) {$y_k$};
\node [below] at (4.3,3) {$x_{k+1}$};
\end{tikzpicture}
\caption{Constructing the sequence in the proof of Lemma \ref{safepoint} from the cumulative distribution function.}\label{cdf}
\end{center}
\end{figure}
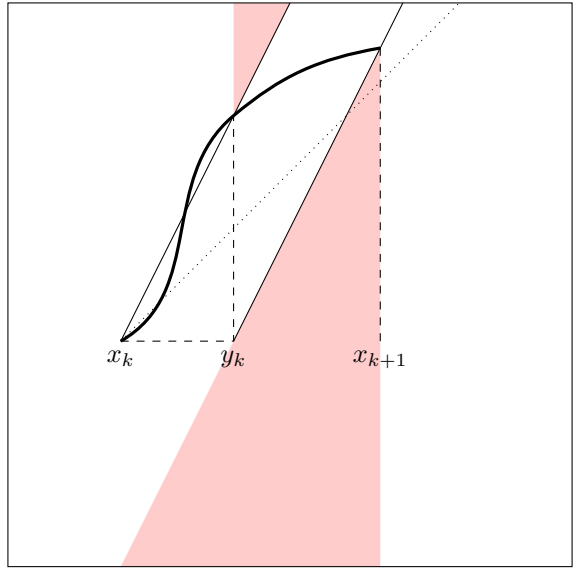

By assumption, there exists some $y>x_k$ such that $\Pr[x_k,y]\ge2(y-x_k)$. Define $y_k=\sup\{y:\Pr[x_k,y]\ge2(y-x_k)\}$, the largest such value. By continuity, we must have $\Pr[x_k,y_k]=2(y_k-x_k)$, which implies that this is where the cdf crosses the solid line with slope 2 out of $x_k$ in Figure \ref{cdf}. Then for any $\eta>0$, we must have $\Pr[x_k,y_k+\eta]<2(y_k+\eta-x_k)=\Pr[x_k,y_k]+2\eta$, so $\Pr[y_k,a_k+\eta]<2\eta$, satisfying (iv). In Figure \ref{cdf}, this shows that the function does not enter the upper red region.

Now define $x_{k+1}=\inf\{b:\Pr[x_k,b]\le2(b-y_k)\}$. This is depicted in Figure \ref{cdf} as the first intersection with a second solid line parallel to the first. Indeed, this must exist because $1$ is a member of the set. This takes a few steps to verify: First, $\Pr[x_k,1]\ge\Pr[x_k,y_k]=2(y_k-x_k)$. Second, $\Pr[x_k,1]=1-\Pr[0,x_k]\le1-x_k$ by the inductive step property (i). Doubling the second inequality and subtracting the first (which aligns the inequalities), $\Pr[x_k,1]\le2(1-x_k)-2(y_k-x_k)=2(1-y_k)$, as claimed. In Figure \ref{cdf}, this shows that the upper right corner is still below the second solid line, so the function must cross it somewhere.

Again, continuity implies that $\Pr[x_k,x_{k+1}]=2(x_{k+1}-y_k)$, so $x_{k+1}>y_k$ as claimed. Therefore, for $0<\eta\le x_{k+1}-x_k$, $\Pr[x_k,x_{k+1}-\eta]>2(x_{k+1}-\eta-y_k)=\Pr[x_k,x_{k+1}]-2\eta$, so $\Pr[x_{k+1}-\eta,x_{k+1}]<2\eta$, satisfying (iii) for small $\eta$. This also implies that $\Pr[x_k,x_{k+1}]<2(x_{k+1}-x_k)$, so $\Pr[0,x_{k+1}]-2x_{k+1}<\Pr[0,x_k]-2x_k$, satisfying (ii). If $\eta>x_{k+1}-x_k$, $\Pr[x_{k+1}-\eta,x_{k+1}]=\Pr[x_{k+1}-\eta,x_k]+\Pr[x_k,x_{k+1}]\le2(x_k-x_{k+1}+\eta)+2(x_{k+1}-x_k)=2\eta$ by the induction hypothesis, satisfying (iii) for large $\eta$ as well. This shows that the function does not enter the lower red region in Figure \ref{cdf}.

Finally, since $x_{k+1}\ge y_k$, we have $\Pr[x_k,x_{k+1}]\ge\Pr[x_k,y_k]=2(y_k-x_k)$. Averaging this with $\Pr[x_k,x_{k+1}]=2(x_{k+1}-y_k)$, we have $\Pr[x_k,x_{k+1}]\ge x_{k+1}-x_k$. This shows that the function is above the dotted line with slope 1 when it reaches $x_{k+1}$. Therefore, $\Pr[0,x_{k+1}]=\Pr[0,x_k]+\Pr[x_k,x_{k+1}]\ge x_k+(x_{k+1}-x_k)=x_{k+1}$, satisfying (i).

Therefore, our recursive sequence construction goes through. Now, as an increasing sequence in the compact interval $[0,1]$, this sequence must have a limit $L$. We claim that for all $\eta>0$, $\Pr[L-\eta,L],\Pr[L,L+\eta]\le2\eta$, i.e. that $L$ proves the lemma. Let us first prove the first claim. Suppose for some $\eta>0$, $\Pr[L-\eta,L]>2\eta$. Because $x_0=0$ and $x_k\to L$, there exists some $k$ such that $x_k<L-\eta\le x_{k+1}$. By condition (iii) on $x_{k+1}$, we must have $\Pr[L-\eta,x_{k+1}]\le2(x_{k+1}-(L-\eta))$, so $\Pr[x_{k+1},L]\ge2\eta-2(x_{k+1}-(L-\eta))=2(L-x_{k+1})$.

On the other hand, by continuity, we must have $\lim_{k\to\infty}(\Pr[0,x_k]-2x_k)=\Pr[0,L]-2L$. Since we proved (property (ii)) that this sequence is decreasing, we must have $\Pr[0,x_{k+1}]-2x_{k+1}>\Pr[0,L]-2L$. Therefore, $\Pr[x_{k+1},L]<2(L-x_{k+1})$. This is a contradiction, so no such $\eta$ exists.

Finally, suppose that for some $\eta>0$, $\Pr[L,L+\eta]>2\eta$. Since $\Pr[0,x_k]-2x_k$ converges to $\Pr[0,L]-2L$, there exists some $k$ such that $(\Pr[0,x_k]-2x_k)-(\Pr[0,L]-2L)<\Pr[L,L+\eta]-2\eta$. Rearranging, this means that $\Pr[x_k,L+\eta]>2(L+\eta-x_k)$. But since $y_k=\sup\{y:\Pr[x_k,y]\ge2(y-x_k)\}$, this implies that $y_k\ge L+\eta$, which is impossible as $L>y_k$ as the limit of an increasing sequence. So no such $\eta$ exists, and we conclude that $L$ satisfies what we need.\end{proof}
Finally, for the algorithmic statement, since Lemma \ref{safepart} calls Lemma \ref{safepoint} $O(1/\eps)$ times, it suffices to bound how long each call will take:
\begin{cor}If $D$ is a discretely supported distribution with support size $s$, then there exists an $O(s^3)$-time algorithm to find some $x\in[0,1]$ such that $\forall\eta>0$, $\Pr[[x-\eta,x]],\Pr[[x,x+\eta]]\le2\eta$.\end{cor}
\begin{proof}Consider the set of such $x$, which is nonempty by Lemma \ref{safepoint}. We claim this set is closed. Indeed, consider a sequence $x_1,x_2,\dotsc$ that each satisfy the condition and converge to some $x\in[0,1]$. Suppose without loss of generality that $\Pr[[x,x+\eta]]>2\eta+\eps$ for some $\eps>0$. Then there exists some $x_k$ such that $\abs{x_k-x}<\eps/2$. If $x_k>x$, then by assumption, $\Pr[[x,x_k]]\le2(x_k-x)$ and $\Pr[[x_k,x+\eta]]\le2(x+\eta-x_k)$ so $\Pr[[x,x+\eta]]\le2\eta$, a contradiction. If $x_k<x$, then since $x_k>x-\eps/2$, $\Pr[x,x+\eta]]\le\Pr[[x_k,x_k+\eta+\eps/2]]\le2(\eta+\eps/2)$, a contradiction. So indeed, $x$ must satisfy the condition and the set of such $x$ is closed.

Therefore, this set has a maximum. At this maximum, we must have at least one equality; suppose first that $\Pr[[x,x+\eta]]=2\eta$ for some $\eta>0$. Since $D$ is discretely supported and for any $\eps>0$, $\Pr[[x,x+\eta-\eps]]\le2(\eta-\eps)<2\eta=\Pr[[x,x+\eta]]$, $x+\eta$ must be in the support. Similarly, if $\Pr[[x-\eta,x]]=2\eta$, $x-\eta$ must be in the support.

Let the support of $D$ be $d_1<d_2<\dotsb<d_s$. The tight interval $[x-\eta,x]$ or $[x,x+\eta]$ contains some consecutive subset of these $d_i,d_{i+1},\dotsc,d_j$, and either $d_i=x-\eta$ or $d_j=x+\eta$. Moreover, the sum of the weights at these points is exactly $2\eta$, so the two possible values for $x$ can be determined from the subset. There are $O(s^2)$ possible subsets of consecutive points, so $O(s^2)$ possible $x$'s to check. Moreover, for each candidate $x$, we only have to check $s$ possible values of $\eta$ ($\abs{x-d_i}$ for $i=1,\dotsc,s$). Therefore, the brute force algorithm of checking all such $x$ takes $O(s^3)$ time to find one that works.\end{proof}

\end{document}